\documentclass[12pt]{article}

\usepackage{blindtext}
\usepackage{comment}

\usepackage{algorithm}
\usepackage[noend]{algpseudocode}
\usepackage{amsmath, amssymb, amsthm}
\usepackage{graphicx}
\usepackage{subcaption} % For subfigure labeling (a), (b), etc.)
\usepackage{xcolor}
\usepackage[T1]{fontenc}
\usepackage{relsize}
\usepackage{array}
\usepackage{multirow}
\usepackage[hidelinks]{hyperref}
\usepackage{geometry}
\usepackage{natbib}
\usepackage{authblk}

\newcommand{\R}{\mathbb R}
\newcommand{\E}{\mathbb E}

\newcommand{\s}{\mathcal S}
\newcommand{\A}{\mathcal A}
\newcommand{\B}{\mathcal B}
\newcommand{\h}{\mathcal H}
\newcommand{\T}{\mathcal T}

\newcommand{\lb}[1]{\underline{#1}}
\newcommand{\ub}[1]{\overline{#1}}

\DeclareMathOperator*{\argmin}{argmin}
\DeclareMathOperator*{\argmax}{argmax}
\DeclareMathOperator{\given}{\, | \, }
\DeclareMathOperator*{\eq}{=}

\usepackage{lastpage}
\newtheorem{theorem}{Theorem}[section]
\newtheorem{lemma}[theorem]{Lemma}
\newtheorem{corollary}[theorem]{Corollary}
\newtheorem{proposition}[theorem]{Proposition}

\begin{document}

\title{Distributional Reinforcement Learning via the Cram\'er Distance}

\author[1]{V. Aziz}
\author[1]{I. Nowak}
\author[2]{E.M.T. Hendrix}
\affil[1]{HAW Hamburg, 
\texttt{vaziz.edu@gmail.com}}
\affil[1]{HAW Hamburg, \texttt{ivo.nowak@haw-hamburg.de}}
\affil[2]{Universidad de M\'alaga, \texttt{eligius@uma.es}}

% \author{
% \name V. Aziz \email vaziz.edu@gmail.com \\
%        \addr Institute for Mechatronics Engineering \& Cyber-Physical Systems IMECH.UMA
%        \\
%        Universidad de M\'alaga
%        %(and HAW Hamburg)
%        \\
%        29071 M\'alaga
%        \AND
%        \name I. Nowak \email ivo.nowak@haw-hamburg.de \\
%        \addr Department Maschinenbau und Produktion \\
%        HAW Hamburg\\
%        20099 Hamburg
%        \AND
%        \name E.M.T. Hendrix \email eligius@uma.es \\
%        \addr Computer Architecture\\
%        Universidad de M\'alaga\\
%         29071 M\'alaga
% }

\maketitle

\begin{abstract}
This paper explores the application of the Soft Actor-Critic (SAC) algorithm within a Distributional Reinforcement Learning setting and introduces an implementation of such algorithm named 
Cram\'er-based Distributional Soft Actor-Critic (C-DSAC).  
The novel approach employs distributional reinforcement learning to represent state-action values, and minimizes the squared Cram\'er distance for learning the distribution.
Empirical results across various robotic benchmarks indicate that our algorithm surpasses the performance of baseline SAC and contemporary distributional methods, with the performance advantage becoming increasingly pronounced in high-complexity environments.
To explain the efficiency of the new approach, we conduct an analysis showing that its superior performance is partly due to \textit{confidence-driven} Q-value updates:
High-variance target distributions (low confidence in target) lead to 
more conservative model updates, thereby attenuating the impact of overestimated values. 
This work deepens the understanding of distributional reinforcement learning, offering insights into the algorithmic mechanisms governing convergence and value estimation.
\end{abstract}
% \begin{keywords}
%   distributional reinforcement learning, Cram\'er metric, soft actor critic, robotics,  overestimation bias
% \end{keywords}

\section{Introduction}
\label{sec:intro}
In science and technology, deep reinforcement learning algorithms have been successfully applied to various problems. A notable example is AlphaTensor \citep{Fawzi2022DiscoveringFM}; a framework/agent exhibiting superhuman performance for automating algorithmic discoveries in matrix multiplication. 
Reinforcement learning \citep{Sutton1998intro} is used (among other techniques) in contemporary large language models for alignment \citep{Ouyang2022TrainingLM} and  improving model reasoning \citep{zhang2025surveyreinforcementlearninglarge}.
In robotic tasks, this approach has been explored for data-driven generation of optimal control \citep{Gu2016DeepRL}. Nevertheless, applying reinforcement learning presents significant challenges, primarily due to the high sample complexity and the limited amount of data available in real-world scenarios \citep{Paduraru2021ChallengesOR}. 
What has begun as a promising approach to learn discrete control policies \citep{Mnih2013PlayingAW}, and has since evolved to address general problems and mitigate various limitations  \citep{Hasselt2015DeepRL, Lillicrap2015ContinuousCW, Schulman2015TrustRP, Schulman2017ProximalPO, Fujimoto2018AddressingFA, Haarnoja2018SoftAO}, appears to have reached a developmental standstill in model-free, sample-efficient real-world control.

A more recent approach by \cite{Bellemare2017ADP,bellemare2023DRLbook}, which expands and formalizes the distributional perspective on reinforcement learning,  has revitalized interest in the field, including among researchers from other disciplines \citep{Muller2021DistributionalRL}.

In this paper, we investigate the distributional approach to reinforcement learning and introduce a novel method for variance-inverse gradient weighting using the squared Cram\'er distance. 
We present Cramér-based Distributional Soft Actor-Critic (C-DSAC), a framework that integrates maximum-entropy objectives into a distributional reinforcement learning setting.
Moreover, we derive corresponding formulas for implementation with neural networks and present experimental data demonstrating state-of-the-art performance across tested environments.
In order to explain the efficiency of the new approach, a thorough analysis of the algorithm's dynamics is conducted. 
The focus is on the algorithm's behavior in presence of approximation errors and system noise and equations for its value gradients are derived showing that the adaptation to Q-values is slower at state action pairs where the variance of the distributional Q-model is large, i.e. confidence in its value is low. Furthermore, we emphasize how these equations reveal the algorithm's inherent mechanism for mitigating overestimation bias. 

This paper is organized as follows. Section \ref{sec:relatedwork} reviews the relevant literature. Section \ref{sec:background} introduces basic concepts and notation. Section \ref{sec:concepts} describes the theoretical basis of the C-DSAC algorithm and its properties. In Section \ref{sec:impl}, formulas for implementation with neural networks are derived. Section \ref{sec:num} provides a numerical illustration to showcase the performance of C-DSAC in comparison with the SAC algorithm. Section \ref{sec:concl} summarizes our findings and outlines potential future research directions. 

%---------------------------------------------------------------------------------
\section{Related Work}
\label{sec:relatedwork}
Bellemare et al. \cite{Bellemare2017ADP} provided the first formal analysis of the distributional approach in reinforcement learning, in which the objective is to learn the return distributions rather than their expectations.
They employed the Wasserstein metric to measure distances between return distributions, established convergence guarantees for policy evaluation, and analyzed distributional policy improvement.
It was shown that under mild assumptions, distributional policy evaluation and improvement converged to the optimal policy. 
This work culminated in the C51 algorithm, a critic-only method for discrete action spaces. However,  C51 relies on a heuristic projection step that lacks theoretical alignment with the distributional Bellman operator.
In \cite{Duan2020DistributionalSA} and \cite{Ma2020DistributionalSA}, the concepts of SAC and Distributional Reinforcement Learning were applied simultaneously for the first time.  
The former method relies on the Kullback-Leibler divergence as a probability metric.  This conflicts with the properties required for distributional policy evaluation, as the distance becomes infinite if the distributions have dissimilar  supports. To address this issue, the authors applied aggressive clipping, which negatively impacts performance. Moreover, the authors of this work were unable to reproduce the results using their provided software. 
The work in \cite{Ma2020DistributionalSA} extended the SAC framework by incorporating quantile-regression-based distributional reinforcement learning \citep{Dabney2018ImplicitQN}.
This integration spans several quantile-based approaches, ranging from standard QR-DQN \citep{Dabney2017DistributionalRL} to more complex architectures like IQN \citep{Dabney2018ImplicitQN}  and "Fully Parameterized Quantile Function" \citep{yang2019fqf}. 
The authors reported superior performance compared to common reinforcement learning algorithms in the MuJoCo benchmark environments with the IQN configuration.
While the reliance on IQN entails an expanded parameter set via quantile embedding, it is shown that such complexity is redundant for the novel approach presented in this work. The results indicate that C-DSAC’s inherent risk-aversion enables superior performance through a more streamlined architecture.
The work of \cite{Lhritier2021ACD} focuses on the Cram\'er distance "in the setting of fixed quantile levels" and does not present an algorithm for continuous control.
Nam et al. \cite{nam21a} similarly utilize the squared Cram\'er distance within a Distributional Reinforcement Learning framework, primarily to enhance the stability of on-policy algorithms. Their core contribution, the Sample-Replacement ($SR(\lambda)$) algorithm, establishes a principled distributional generalization of the multi-step $\lambda$-return. In contrast, the present work pivots from on-policy stability to the distinct challenge of overestimation bias inherent in off-policy $TD(0)$ learning.

\section{Background}
\label{sec:background}
Our work builds on maximum-entropy and distributional reinforcement learning. In the following sections, we provide an overview of these topics and introduce our notation. 

\subsection{Reinforcement Learning}
\label{sec:rl}
In reinforcement learning, the interaction of an agent with its environment is modelled as a MDP $\mathcal{M} = (\s,\A, r, P, \gamma)$. In this context, $\s$ is the finite state space, $\A$ denotes the finite action space, $r: \s \times \A \rightarrow [r_{min}, r_{max}]$ is the bounded reward function,
% \att{\footnote{The set of all probability measures on a measurable space $(\Omega, \Sigma)$ is denoted by $Pr(\Omega)$.}}
$P: \s \times \A \rightarrow Pr(\s)$
defines the transition probability dynamics and $\gamma \in (0,1)$ denotes the discount factor.
At each time step $t$, an agent executes an action $a_t \in \A$ from state $s_t \in \s$ according to a policy $\pi$ and receives a reward $r(s_t,a_t)$, transitioning into the next state $s_{t+1} \sim P(\cdot \mid s_t, a_t)$. 
An episodic MDP considers to reach  a terminal state at  $T$  within finite time. 
In this work, the policy is treated as a stochastic quantity, modeled as a probability distribution over the action space  $\pi: \s \rightarrow Pr(\A)$ and $\Pi$ denotes the set of such policies. 
The action is sampled from the policy, i.e. $a_t \sim \pi(\cdot \mid s_t)$.  
The general goal is to find an optimal policy $\pi^*$ that maximizes 
the expected discounted return along a trajectory of interactions, 
\begin{equation}
\begin{aligned}
\label{eq_policyobjective_rl_random}
J(\pi) &:= \E_{\substack{s_0 \sim d_0 \\a_{t} \sim \pi(\cdot \mid s_t) \forall t=0, \hdots, T-1, \\ s_{t+1} \sim P(\cdot \mid s_t,a_t) \forall t=0, \hdots, T-1}} \left[ \sum_{t=0}^{T-1}{\gamma^{t} r(s_t,a_t)} \right] \\
&:= \E_{\pi, P} \left[ \sum_{t=0}^{T-1}{\gamma^{t} r(s_t,a_t)} \right],
\end{aligned}
\end{equation}
where $s_T$ is an absorbing state. 

% \att{i.e., $\pi^* = \argmax_{\pi\in\Pi}J(\pi)$; an optimal policy $\pi^*$ must not be unique.} 
The initial state is obtained by drawing from the distribution $s_0 \sim d_0$.
Assuming discrete spaces, equation \eqref{eq_policyobjective_rl_random} can be written more explicitly as 
\begin{equation}
\label{eq:policyobjevtice_rl_summation}
J(\pi) = \sum_{s \in \s}{\sum_{a \in \A}{\sum_{t=0}^{T-1}}{\Pr(s_t=s \mid \pi, P, s_0 \sim d_0) \pi(a \mid s)\gamma^tr(s,a)}}
\end{equation}
In this setting, the episode has a length of T and thus $J$ optimizes a discounted return in a finite MDP. 
\begin{comment}
Unlike an infinite-horizon MDP, where an upper bound
of $J$ is the maximum value of the geometric series ad infinitum, $\max_{\pi}{J(\pi)} \leq \frac{1}{1 - \gamma} r_{max}$, an upper bound of $J$ on a finite MDP is $\max_{\pi}{J(\pi)} \leq \frac{1 - \gamma^{T}}{1- \gamma} r_{max}$. This raises a concern, as distant future rewards may be effectively ignored. The problem can be avoided by adjusting the discount factor to align with the average episode length. 
Additionally, heuristics such as a hyperparameter search or adherence to established values can be employed \cite{eimer2023hyperparametersinrl}. In that way, accurate value estimations across different environments can be ensured. 
\end{comment}

The optimal policies can be approximated by policy iteration, a process that incorporates policy evaluation and improvement using Q-values. 
A Q-value induced by a policy $\pi$ is defined as
\begin{equation}
\begin{aligned}
\label{eq:q_ac}
Q^{\pi}(s_k,a_k) &:= \E_{\substack{a_t \sim \pi(\cdot \mid s_t) \forall t=k+1, \hdots, T-1 \\ s_{t+1} \sim P(\cdot \mid s_t,a_t) \forall t=k, \hdots, T-1}} \left[ \sum_{t=k}^{T-1}{\gamma^{t-k}r(s_t,a_t)} \mid s_k ,a_k\right] \\
&:= \E_{\pi, P} \left[ \sum_{t=k}^{T-1}{\gamma^{t-k} r(s_t,a_t)} \mid s_k, a_k\right],
\end{aligned}
\end{equation}
and satisfies the recursive Bellman equation
\begin{equation}
    \label{eq:ac_bellman}
    Q^{\pi}(s_t,a_t)= r(s_t, a_t) + \gamma \E_{\substack{ a_{t+1} \sim \pi(\cdot \mid s_{t+1}), \\ s_{t+1} \sim  P(\cdot \mid s_t,a_t)}}[Q^{\pi}(s_{t+1}, a_{t+1})].
\end{equation}
Based on \eqref{eq:ac_bellman}, we define the Bellman operator $\mathcal{T^{\pi}}$,
\begin{equation}
\label{eq:BO2}
\T^{\pi}Q(s_t, a_t) := r(s_t,a_t) + \gamma \E_{\substack{s_{t+1} \sim P(\cdot \mid s_t,a_t), \\ a_{t+1} \sim \pi(\cdot \mid s_{t+1})}}\left[Q(s_{t+1}, a_{t+1})\right].
\end{equation}
For policy evaluation, it can be shown that 
the sequence of value functions $Q_{k+1} := \T^\pi Q_k$, starting from some arbitrary $Q_0$, will exponentially quickly converge to $Q^{\pi}$ as $k$ increases, see Lemma \ref{lem:policyevaluation_rl}.

Policy improvement, Lemma \ref{lem:policyimprovement_rl}, involves exploitation of knowledge about the environment and updating towards higher values,  i.e.
\begin{equation}
\label{eq:policyupdate_rl}
\pi_{k+1}(s_t) = \argmax_{\pi \in \Pi}\E_{a_t\sim \pi(\cdot \mid s_t)} {Q^{\pi_k}(s_t, a_t)}.
\end{equation}

In that manner, the policy will be non-decreasing. The estimation of the state-action value can be improved by better exploration of the state-space. 

Convergence towards the optimal policy $\pi^*$ is guaranteed under policy iteration, see Lemma \ref{lem:policyiteration_rl}.

\subsection{Maximum Entropy Reinforcement Learning}
\label{sec:hrl}
In this setting, the standard objective in Eq.\eqref{eq_policyobjective_rl_random} is augmented with policy entropy to improve exploration 
\begin{equation}
\label{eq:policyobjective_hrl}
\begin{aligned}
    J_h(\pi) &:= \E_{\substack{s_0 \sim d_0 \\a_{t} \sim \pi(\cdot \mid s_t)  \forall t=0, \hdots, T-1, \\ s_{t+1} \sim P(\cdot \mid s_t,a_t) \forall t=0, \hdots, T-1}} \left[ \sum_{t=0}^{T-1}{\gamma^{t} \left( r(s_t,a_t) + \alpha \mathcal{H}(\pi(\cdot \mid \ s_t)) \right]} \right] \\
    &:= \E_{\pi, P} \left[ \sum_{t=0}^{T-1}{\gamma^{t} (r(s_t,a_t) + \alpha \mathcal{H}(\pi(\cdot \mid \ s_t)))} \right],
\end{aligned}
\end{equation}
where
 $\h(\pi(\cdot|s_t)):=\E_{a_t\sim \pi(\cdot|s_t)}[-\log \pi(a_t|s_t)]$
\citep{Haarnoja2018SoftAO}.

Soft policy improvement involves the information projection and update towards the
exponential of the new Q-value,
\begin{equation}
    \label{eq:policy_improvement_sac}
    \pi_{k+1}(\cdot \mid s_t) = \argmin_{\pi(\cdot \mid s_t)}d_{kl} \left( \pi(\cdot \given s_t) \mathrel{\Big|\Big|} \frac{\exp(\frac{1}{\alpha} Q^{\pi_k}(s_t, \cdot))}{W^{\pi_k}(s_t)} \right),
\end{equation}
where $\alpha > 0$, $d_{kl}$ is the Kullback-Leibler divergence function and $W^{\pi_k}(s_t) = \sum_i^{|\A|}{\exp({\frac{1}{\alpha}Q^{\pi_k}(s_t,a_i)})}$ denotes the partition  function normalizing the distribution. $W^{\pi_k}(s_t)$ can be ignored since it does not contribute to the new policy.  
The parameter $\alpha$ represents the impact of policy stochasticity on the state-action value. 

Analogous to the definition of the Bellman operator for standard reinforcement learning, a modified Bellman operator $\T^{\pi}_h$ 
is defined as
\begin{equation}
\label{eq:BOH}
\T_h^{\pi} Q(s_t, a_t):= r(s_t,a_t) + \gamma \E_{\substack{ s_{t+1} \sim P(\cdot \mid s_t,a_t)\\a_{t+1} \sim \pi(\cdot \mid s_{t+1})}}[Q(s_{t+1}, a_{t+1}) + \alpha \mathcal{H}(\pi(\cdot \mid s_{t+1}))].
\end{equation}
The soft state-action value is given by
$$V_h^{\pi}(s_t) := \E_{a_t \sim \pi(\cdot \mid s_t)}[Q^{\pi}(s_t, a_t) + \alpha \mathcal{H}(\pi(\cdot \mid s_{t}))].$$ 
As shown in the Lemma \ref{lem:policyevaluation_hr}, the operator $\T_h^{\pi}$ is a $\gamma$-contraction, analogous to $\T^{\pi}$.

Soft policy iteration is guaranteed to converge to the optimal policy $\pi^*$ (Lemma \ref{lem:policyiteration_sac}), proceeding through alternating steps of soft policy evaluation under the soft Bellman operator $\T_h^{\pi}$ and soft policy improvement (Lemmas \ref{lem:policyimprovement_hrl_i} and \ref{lem:policyimprovement_hrl_ii}).

\subsection{Distributional Reinforcement Learning}
Distributional reinforcement learning (DistRL) shifts the focus from modeling expected returns, as in traditional value-based methods, to capturing the entire distribution of returns. 
Therefore, the state-action value function is treated as random variable. The distribution over the return is defined as
\begin{equation}
\begin{aligned}
\label{eq:z_ac}
Z^{\pi}(s_k,a_k) &:= \sum_{t=k}^{T-1}{\gamma^{t-k} r(s_t, a_t)}, (a_t \sim \pi(\cdot \mid s_t) \forall t=k+1 \hdots T-1, \\ 
& \quad \quad \quad  s_{t+1} \sim P(\cdot \mid s_t,a_t) \forall t=k \hdots T-1 ) \mid  s_k, a_k.
\end{aligned}
\end{equation}
If the reward is treated as a deterministic quantity, then stochasticity in the return distribution arises from the transition kernel and the stochastic policy. Although DistRL was devised originally for aleatoric uncertainty \citep{Bellemare2017ADP}, i.e. the reward is explicitly treated as a random variable, experiments show that distributional variants outperform their scalar counterparts even in environments with deterministic rewards, see Section \ref{sec:num}. 
The expectation recovers the state-action value, 
\begin{equation}
\label{eq:q_z_relation}
Q^{\pi}(s_k,a_k) = \E[Z^{\pi}(s_k,a_k)].
\end{equation}
Despite prior investigations into the mechanisms explaining accelerated learning \citep{Bellemare2017ADP, Wang2023TheBO, Wang24MoreBenefits}, a comprehensive understanding remains elusive. We contribute to this discussion by introducing confidence-driven model updates in Section \ref{sec:stability}.
The distributional counterpart of the Bellman operator in \eqref{eq:BO2} is used for distributional evaluation. 
We define it under the Wasserstein metric as 
\begin{equation}
\label{eq:distBellmanOp}
\T^{\pi}_D Z(s_t,a_t) :\eq^d r(s_t, a_t) + \gamma Z(s_{t+1}, a_{t+1})\given s_{t+1} \sim P(\cdot \mid s_t,a_t), a_{t+1} \sim \pi(\cdot \mid s_{t+1}), 
\end{equation}
where $\eq^d$ indicates that the random variable on both sides are distributed according to the same law.
It can be shown that $\T_D^{\pi}$ is a $\gamma$-contraction \citep{Bellemare2017ADP}, a property needed for distributional policy evaluation.
Theorem \ref{theo:policyevaluation_cdsac} establishes that the $\gamma$-contraction property of the distributional Bellman operator extends to the squared Cram\'er distance.

Bellemare et al. \cite{Bellemare2017ADP} analyze the distributional Bellman optimality operator, demonstrating that convergence to the optimal distribution holds only under the condition of a unique optimal policy or a total ordering on the set of optimal policies. 

\subsection{The Squared Cram\'er Distance}
\label{sec:cramer_reg}
Choosing an appropriate metric to quantify the distance between random variables $U$ and $V$ is critical for policy evaluation and improvement in distributional reinforcement learning. 

%Bellemare et al. 
Bellemare et al. \citep{Bellemare2017TheCD}  noted that a metric must be both "ideal" according to Zolotarev et al. \cite{Zolotarev1976MetricDistances} and suitable for gradient descent methods to be effective in machine learning with distributions.
The squared $L_2$ Cram\'er distance possesses such ideal properties and is suited for machine learning.

To formalize this, consider the proper $L_2$ Cram\'er metric.
For two random variables $U$ and $V$ it is defined as
\begin{equation}
    \label{eq:cramermetric}
    d_{c}(U, V) := \left(\int_{- \infty}^{\infty}{(F_U(x) -F_V(x))^2 dx} \right)^{\frac{1}{2}},
\end{equation}
where $F_U$ and $F_V$ are the cumulative distribution functions of $U$ and $V$ respectively. 
Note that the distance is formally between two laws of the random variables, \\
$d_c(\mathcal{L}(U), \mathcal{L}(V))$, but the common shorthand $d_c(U,V)$ is used to avoid convoluted notation.
Ideal metrics require to be sum invariant and scale sensitive.
The Cram\'er metric is sum invariant, therefore it holds
\begin{equation}
\label{eq:sum_invar_rand}
\tag{C1}
d_c(A+U, A+V) \leq d_c(U,V),
\end{equation}
where $A$ is a random variable independent of $U$ and $V$. 
Treating $A$ as a degenerate random variable with $P(A=a)=1$, this property can be also stated in terms of a constant 
\begin{equation}
\label{eq:sum_invar_const}
d_c(a+U, a+V) \leq d_c(U,V).
\end{equation}
Scale sensitivity, the second required property, ensures that for $k>0$, 
\begin{equation}
\label{eq:scalesensitivity_cr}
\tag{C2}
d_c(kU, kV) = |k|^{\frac{1}{2}}d_c(U,V).
\end{equation}

\begin{lemma}
The  Cram\'er metric  $d_c$ satisfies  properties \eqref{eq:sum_invar_rand} and \eqref{eq:scalesensitivity_cr}. 
\end{lemma}
\begin{proof}
See Theorem 2 of \cite{Bellemare2017TheCD}.
\end{proof}

A third property is needed to  guarantee unbiased sample gradients. Within the class of $L_p$ Cramér metrics, only the squared Cramér distance ($p=2$), also called the energy distance, satisfies this requirement.
The energy distance is expressed as
\begin{equation}
    \label{eq:energy_distance}
    d_{e}(U, V) := d_c^2 =\int_{- \infty}^{\infty}{(F_U(x) -F_V(x))^2 dx},
\end{equation}
and has the unbiased sample gradient property
\begin{equation}
\label{eq:unbiasedsamplegradients}
\tag{C3}
\E_{\mathbf{X_m} \sim \Gamma}\nabla_{\theta}d_e(\hat{\Gamma}_m, Z_{\theta}) = \nabla_{\theta}d_e(\Gamma, Z_{\theta}),
\end{equation}
where $\mathbf{X}_m$ is a vector with elements $X_1, \ldots, X_m$ drawn from a Bernoulli distribution $\Gamma$, and $\hat{\Gamma}_m := \frac{1}{m} \sum_{i=1}^m \delta_{X_i}$ is
an approximate distribution formed by the samples with $\delta$ being Dirac functions at values $X_i$.  

\begin{lemma}
\label{lem:energy_properties}
The  energy distance  $d_e$ satisfies properties \eqref{eq:sum_invar_rand},\eqref{eq:unbiasedsamplegradients}, and 
\begin{equation}
\label{eq:scalesensitivity_en}
\tag{C4}
d_e(kU, kV) = kd_e(U,V),
\end{equation}
with $k>0$.
\end{lemma}
\begin{proof}
See Appendix A3 of \cite{Bellemare2017TheCD}. 
\end{proof}
By parity of reasoning with the derivation of Equation \eqref{eq:sum_invar_const}, this identity extends to the case of constant values.

The following Lemma provides the last property necessary to build the algorithm's theoretical framework.
\begin{lemma}   
    Let $d_e$ be the energy distance as defined in \eqref{eq:energy_distance}.
    For any collection of distributions $\{\mu_i, \nu_i\}_{i=1}^k$ and non-negative weights $\{w_i\}$ such that $\mu=\sum_i{w_i\mu_i}$, $\nu=\sum_i{w_i \nu_i}$ and $\sum_i{w_i}=1$, the following inequality holds, where for notational consistency the energy distance is written in terms of variables $U \sim \mu$, $V \sim \nu$ and $U_i \sim \mu_i$, $V_i \sim \nu_i$, respectively
    \begin{equation} 
    \label{eq:energy_convexity} 
    \tag{C5} 
    d_e(U, V) \leq \sum_{i=1}^k{w_i d_e(U_i, V_i)}. 
    \end{equation}
\end{lemma}
\begin{proof}
    By Theorem 22 of Sejdinovic et al.  \cite{sejdinovic2013equivalence}, the energy distance equals a constant multiple of the RKHS norm of kernel mean embeddings. Using linearity and standard results in the RKHS yields \eqref{eq:energy_convexity}.
\end{proof}
%---------------------------------------------------------------------------------

%----------------------------------------------------------------

\section{The C-DSAC Framework}
\label{sec:concepts}
Building on the principles of distributional and maximum-entropy reinforcement learning, we introduce the theoretical framework of our proposed algorithm and offer a comprehensive theoretical justification for its approach in this section.

\paragraph{Distributional Maximum Entropy Reinforcement Learning}
%The aim is to find a policy $\pi^*=\argmax_{\pi\in \Pi} J_H(\pi)$ that maximizes the entropy-augmented expectation over the distributional returns.
The C-DSAC objective is equivalent to \eqref{eq:policyobjective_hrl}, with the distributional soft state-action value made explicit,
\begin{equation}
    \label{eq:poliycobjective_cdsac}
    J_h(\pi) = \E_{\substack{s_0 \sim d_0}} \left[ V_h^{\pi}(s_0) \right],
\end{equation}
where
$$V_h^{\pi}(s_t) := \E_{a_t \sim \pi(\cdot \mid s_t)}[\E[Z_h^{\pi}(s_t,a_t)] - \alpha \log{\pi(a_t \mid s_t)}]$$
and
$$
\begin{aligned}
Z_h^{\pi}(s_t,a_t) &= r(s_t,a_t) + \gamma (Z_h^{\pi}(s_{t+1}, a_{t+1}) - \alpha \log{\pi(a_{t+1} \mid s_{t+1})}) \\
&\mid s_{t+1} \sim P(\cdot \mid s_t,a_t), a_{t+1} \sim \pi(\cdot \mid s_{t+1}).
\end{aligned}
$$
The aim is to find a policy $\pi^*=\argmax_{\pi\in \Pi} J_H(\pi)$.
We chose the negative log probability to quantify entropy. 

\paragraph{Policy Evaluation}
Consider a distributional reinforcement learning setting under the squared $L_2$ Cram\'er distance 
where the distribution over returns $Z$ is utilized.
Then the soft distributional Bellman operator is defined as 

\begin{equation}
    \begin{aligned}
    \label{eq:Tdef}
        \T_H^{\pi} Z(s_t, a_t):\eq^d & r(s_t,a_t) + \gamma (Z(s_{t+1},a_{t+1}) - \alpha \log{\pi(a_{t+1} \mid s_{t+1})}) \\
        & \given s_{t+1} \sim P(\cdot \mid s_t,a_t), a_{t+1} \sim \pi(\cdot \mid  s_{t+1}).
    \end{aligned}
\end{equation}
\begin{theorem}
\label{theo:policyevaluation_cdsac}
Fix $1 \leq p < \infty$ and $(s_t, a_t) \in \s \times \A$. Let $\bar{d}_e(Z_1, Z_2) := \sup_{s,a} d_e(Z_1(s,a), Z_2(s,a))$ for two state-action value distributions $Z_1$ and $Z_2$. Assume that for each $(s,a)$, the distributions have finite first moments to attain completeness of the metric space. Further, assume $|\A| < \infty$, $|\s|<\infty$, $0 \leq \gamma < 1$ and $\alpha \geq 0$. Then $\T_H^{\pi}$ is a $\gamma$-contraction in the energy distance distance, i.e.
\begin{equation}
    \label{eq:contraction1}
    \bar{d}_e(\T_H^{\pi}Z_1 , \T_H^{\pi}Z_2) \leq \gamma
    %^{\frac{1}{2}}
    \bar{d}_e(Z_1, Z_2).
\end{equation}
\end{theorem}

\begin{proof}
The reward can be augmented with the entropy term, making it a random variable.
Let 
$$R_h(s_t,a_t) := r(s_t,a_t) - \gamma \alpha \log{\pi(a_{t+1} \mid s_{t+1})}\given s_{t+1}\sim P(s_t, a_t), a_{t+1} \sim \pi(\cdot \mid s_{t+1}),$$
then
\begin{align*}
    d_e(&\T_H^{\pi}Z_1, \T_H^{\pi}Z_2) \\
    &= d_e \left(R_h(s_t,a_t) + \gamma Z_1(s_{t+1}, a_{t+1}), R_h(s_t,a_t) + \gamma Z_2(s_{t+1}, a_{t+1}) \right) \\
    & \quad \quad \text{where \ } s_{t+1} \sim  P(\cdot \mid s_t,a_t), a_{t+1} \sim \pi(\cdot \mid s_{t+1}). \\
    & \text{By Property (\ref{eq:energy_convexity}) with $w_{t+1} := P(s_{t+1} \mid s_t, a_t) \pi(a_{t+1} \given s_{t+1})$}\\
    & \leq \sum_{s_{t+1}, a_{t+1}}{ w_{t+1} d_e \left(R_h(s_t,a_t) + \gamma Z_1(s_{t+1}, a_{t+1}), R_h(s_t,a_t) + \gamma Z_2(s_{t+1}, a_{t+1}) \right)}.\\
    & \text{Property (\ref{eq:sum_invar_rand}) holds for $d_e$ (Lemma \ref{lem:energy_properties}) and each state-action pair is fixed} \\
    & \leq \sum_{s_{t+1}, a_{t+1}}{w_{t+1} d_e \left( \gamma Z_1(s_{t+1}, a_{t+1}), \gamma Z_2(s_{t+1}, a_{t+1}) \right)}. \\
    & \text{By Property (\ref{eq:scalesensitivity_en})} \\
    &= \gamma \sum_{s_{t+1}, a_{t+1}}{ w_{t+1} d_e \left( Z_1(s_{t+1}, a_{t+1}), Z_2(s_{t+1}, a_{t+1}) \right)}. \\
    & \text{Since $\sum_i{w_i} = 1$,}\\
    &\leq \gamma \sup_{s,a}d_e(Z_1(s,a), Z_2(s,a))\\
    &= \gamma \bar{d}_e(Z_1,Z_2).
\end{align*}

Taking the supremum over $(s,a) \in \s \times \A$ on both sides yields the contraction
$$\bar{d}_e(\T_H^{\pi}Z_1, \T_H^{\pi}Z_2) \leq \gamma \bar{d}_e(Z_1,Z_2).$$

\end{proof}
Since $\T_H^{\pi}$  is a $\gamma$-contraction in maximal form, it follows from the Banach fixed point theorem:
\begin{corollary}
Given $Z_{k+1} :\eq^d \T_H^\pi Z_k$, the series $\{Z_k\}$ converges 
%under the mild conditions of the 
to the fixed point (distribution)
$Z_h^{\pi}$ in the energy distance, i.e.,
\begin{equation}
\label{eq:cramer_contract}
     \lim_{k\to \infty } \bar{d}_e(Z_k,Z_h^\pi)=0
     \quad\text{  with  }\quad
     \T_H^\pi Z_h^\pi \eq^d Z_h^{\pi}. 
\end{equation}
\end{corollary}

\paragraph{Policy Improvement}
Soft policy improvement is achieved via information projection, where the policy is updated toward the Boltzmann distribution defined by the expectation of the distributional soft state-action function $Z_h^{\pi}$, 
\begin{equation}
    \label{eq:policyimprovement_cdsac}
    \pi_{k+1}(\cdot \mid s_t) = \argmin_{\pi(\cdot \mid s_t)} d_{kl}\left(\pi(\cdot\given s_t) \ 
    \mathlarger{||} \ \frac{\exp(\frac{1}{\alpha}\E[Z_h^{\pi_k}(s_t, \cdot)])}{W^{\pi_k}(s_t)}\right).
\end{equation}
The proof follows analogously to the soft policy improvement established in Lemmas \ref{lem:policyimprovement_hrl_i} and \ref{lem:policyimprovement_hrl_ii}, with the distinction that the update is performed with respect to the expectation over the distribution.

\paragraph{Policy Iteration}
The proof parallels the classical reinforcement learning argument, adjusting for the stochastic policy update via an expectation over the action distribution.

\begin{theorem}
\label{theorem:policyiteration_cdsac}
Let the reward be bounded, $|\A| < \infty$, $|\s|< \infty$, $0 \leq \gamma < 1$ and $\alpha >0$. Let $\Pi$ be a set of all stationary, stochastic policies. 
Alternating between exact distributional soft policy evaluation and global soft policy improvement from some initial policy $\pi_0 \in \Pi$ and soft state-action function $\E[Z_h^{(0)}]$, the process converges (in the limit) to the optimal policy $\pi^*$, satisfying $\E[Z_h^{\pi^*}] = \E[Z_h^*]$. 
\end{theorem}
\begin{proof}
At each iteration $k$, the policy $\pi_k$ is fixed during evaluation. By Theorem \ref{theo:policyevaluation_cdsac} and its Corollary, $\T_H^{\pi_k}$ is a $\gamma$-contraction in the energy distance and the evaluation sequence converges to the fixed point $Z_h^{\pi_k}$. Therefore, $\E[Z_h^{\pi_k}]$ is well defined at each iteration.
By the policy improvement argument, the sequence $\{\E[Z_h^{\pi_k}(s_t,a_t)]\}$ is non-decreasing for each pair $(s_t,a_t)$. Because the reward is bounded and $|\s| < \infty$, $\A < \infty$, the sequence converges to some point-wise limit $\bar{L}$.
Policies are updated by the $\text{softmax}$ of the current $\E[Z_h]$ and softmax is continuous, therefore the limit point $\pi^*$ exists and is the softmax of $\bar{L}$, 
i.e. $\pi^*=\text{softmax}(\bar{L})$.
Since $\pi^* \in \Pi$, Theorem \ref{theo:policyevaluation_cdsac} equally guarantees well-definedness in the limit point $\E[Z_h^{\pi^*}]$.
The map $\pi \mapsto \E[Z_h^{\pi}]$ is continuous, hence $\pi_k \rightarrow \pi^*$ implies $\E[Z_h^{\pi_k}] \rightarrow \E[Z_h^{\pi^*}]$ and therefore $\bar{L} = \E[Z_h^{\pi^*}]$. 
It follows $\pi^* = \text{softmax}(\E[Z_h^{\pi^*}])$ and therefore $\E[Z_h^{\pi^*}]$ satisfies the Bellman optimality equation, thus $\E[Z_h^{\pi^*}] = \E[Z_h^*]$.
\end{proof}

The theorems above establish the theoretical foundation for our Cram\'er-based Distributional Soft Actor-Critic (C-DSAC) approach. In Section \ref{sec:impl} we use these insights to construct the C-DSAC algorithm.

%
%----------------------------------------------------------------
\section{Implementation and Analysis}
\label{sec:impl}
In this section, the theoretical results in the previous sections are utilized %our previous work 
to derive  parameterized objective functions suitable for implementation. 
To handle the large state and action spaces in complex tasks, neural networks are used as parameterized function approximators; the aim is to optimize their parameters w.r.t the objective. 
The parameters for the random variable and policy will be denoted by $\theta$ and $\phi$, respectively. 
C-DSAC is implemented as a fixed-moment algorithm, i.e. the underlying distribution is fully characterized by its first two moments.
Each return distribution is obtained by shifting and rescaling a single standardized density. 
It is assumed that the random variable obeys a Gaussian distribution with expectation $Q_{\theta_1}$ and standard deviation $\sigma_{\theta_2}$. 
By constraining the return distribution to this simplified representation, fixed-moment distributional reinforcement learning algorithms operate under strong assumptions about the underlying true distributions of returns.
The fixed-moment approach offers computational simplicity, but it comes at the potential cost of accuracy. The rigid assumptions can (and generally will) lead to biased approximations of the true return distribution. However, this error introduces an intriguing tradeoff between variance and bias in learning. By reducing the complexity of the distribution - requiring fewer parameters for approximation - fixed-moment algorithms can yield lower gradient variance during optimization. The reduction in variance can enhance learning stability and efficiency, making these algorithms appealing in scenarios where computational resources are limited.

To understand the factors contributing to the state-of-the-art performance of C-DSAC, we analyze the impact of the energy distance loss on its training dynamics.
As discussed in the following sections, C-DSAC has an inherent mechanism to counteract overestimation bias.

\subsection{Objective Functions and Algorithm}
\label{sec:objectives}
Formulas are derived for implementation with parameterized function approximators, utilizing neural networks for their high generalization capabilities. Previously observed states and actions are stored in a replay buffer $\B$. Following \cite{Fujimoto2018AddressingFA}, a critic target network with parameters $\bar{\theta}$ is employed for the critic to further reduce approximation errors.

\paragraph{Evaluation}
The neural network parameters of $Z_{\theta}$ are optimized by minimizing the energy distance relative to the soft distributional Bellman target
\begin{equation}
J_Z(\theta) := \E_{\substack{s_t \sim d_{\mu}^{\gamma} \\ a_t \sim \mu(\cdot \mid s_t) \\ s_{t+1} \sim P(\cdot \mid s_t,a_t)}}[d_e(\T_H^{\pi}Z(s_t,a_t), Z_{\theta}(s_t,a_t))],\end{equation}
where $d_{\mu}^{\gamma}=\frac{1 - \gamma}{1-\gamma^T}\sum_{t=0}^{T-1}{\gamma^t Pr(s_t=s \mid \mu)}$ is the state-visitation distribution induced by the stochastic behavior policy $\mu$.

In practice, gradients are estimated using a mini-batch $\hat{\B}$ sampled uniformly from the replay buffer
\begin{equation}
\label{cdsac_policyevaluation_objective}
    \hat{\nabla}_{\theta} J_Z(\theta) = \frac{1}{|\hat{\B}|}\sum_{(s_t, a_t, r_t, s_{t+1}) \in \hat{\B}}[\nabla_{\theta} d_e(\hat{\T}_H^{\pi} Z_{\bar{\theta}}(s_t,a_t),
    Z_{\theta}(s_t,a_t))], 
\end{equation}
where the empirical soft distributional Bellman operator for a fixed $s_{t+1}$ is defined as
$$\hat{\T}_H^{\pi} Z_{\bar{\theta}}(s_t, a_t) := r(s_t,a_t) + \gamma (Z_{\overline{\theta}}(s_{t+1}, a_{t+1}) - \alpha \log{\pi(a_{t+1} \mid s_{t+1})}) \mid a_{t+1} \sim \pi(\cdot \mid s_{t+1})$$ and $\bar{\theta}$ are frozen parameters (for clarity, $\bar{\theta}$ will be omitted in the subsequent sections).
The target network parameters are not updated via gradient descent; instead, it tracks the online network through Polyak averaging. 

The probability density function is Gaussian denoted by
 $$\varphi_{Q, \sigma}(x)=\frac{1}{\sigma \sqrt{2\pi}} e^{ -\frac{1}{2} \left(\frac{(x-Q)}{\sigma}\right)^2}$$
regarding a mean value   $Q\in [\lb Q,\ub Q]\subset \R$ and a variance $\sigma\in [\lb \sigma,\ub \sigma]\subset \R_+ $.  For terminal states, the standard deviation of the target is set to $0$. 
The cumulative distribution function of $\varphi_{Q, \sigma}$ is denoted by 
$$
F_{Q, \sigma}(x)=\int_{-\infty}^{x}\varphi_{Q, \sigma}(t)dt.
$$
Let $Q_{\theta_1}(s_t, a_t)$ and $\sigma_{\theta_2}(s_t, a_t)$ be neural network models of $Q$ and $\sigma$ over ${\cal S}\times {\cal A}$ respectively. Then $Z_{\theta}(s_t, a_t)$ denotes a parameterized Gaussian value distribution with parameters $\theta = (\theta_1, \theta_2)$, defined by the distribution function 
$$f_{Z_{\theta}(s_t, a_t)}=\varphi_{Q_{\theta_1}(s_t, a_t), \sigma_{\theta_2}(s_t, a_t)}. $$

\paragraph{Improvement}
The policy improvement follows from the information projection equation \eqref{eq:policyimprovement_cdsac}.  
Similar to \cite{Haarnoja2018SoftAO} and \cite{Kingma2013AutoEncodingVB}, 
a reparameterized policy is used to obtain the actions,
$a_t = g_{\phi}(s_t; \xi_t)$, where $\xi_t \sim \mathcal{N}(0,1)$ and $\mathcal{N}(0,1)$ is the standard normal distribution. The policy $\pi_{\phi}(\cdot \mid s_t)$ is defined as the pushforward of $g_{\phi}(s_t;\cdot)$ under the law of $\xi_t$.
The policy loss can then be defined as 
\begin{equation}
\label{eq:cdsac_policyimprovement_objective}
    J_{\pi}(\phi) := \E_{\substack{s_t \sim d_{\mu}^{\gamma}, \\ \xi_t \sim \mathcal{N}(0,1) \\ a_t=g_{\phi}(s_t;\xi_t)}} [\alpha \log{\pi_{\phi}(a_t \mid s_t)} - \E[Z_{\theta}(s_t, a_t)]],
\end{equation} 
where $d_{\mu}^{\gamma}=\frac{1 - \gamma}{1-\gamma^T}\sum_{t=0}^{T-1}{\gamma^t Pr(s_t=s \mid \mu)}$ is the state-visitation distribution induced by the stochastic behavior policy $\mu$. The gradient of the empirical loss regarding a mini-batch $\hat{\B} \subset \B$  is thus
\begin{equation}
    \begin{aligned}
    \label{actor_gradient}
    \hat{\nabla}_{\phi}J_{\pi}(\phi) =& \frac{1}{|\hat{\B|}} \sum_{\substack{s_t \in \hat{\B}, \\\xi \sim \mathcal{N}(0,1)}}\alpha \partial_{\phi}\log{\pi_{\phi}(a_t \given s_t)}|_{a_t} +\\
    &(\alpha \nabla_{a_t} \log{\pi_{\phi}(a_t \given s_t)} - \nabla_{a_t} \E[Z_{\theta}(s_t, a_t)])\nabla_{\phi}g_{\phi}(s_t;\xi).
    \end{aligned}
\end{equation}
Similar to the distributional value function, the action distribution is modeled as a Gaussian. Additionally, 
the law of $\xi$ for each action component is set to be a standard Gauss distribution.

\paragraph{Algorithm}
The C-DSAC algorithm is presented in Algorithm \ref{alg:cdsac}, with the Cram\'er gradient update described in line \ref{line:theta}.  

\begin{algorithm} 
\caption{C-DSAC (Cram\'er-based Distributional Soft Actor-Critic)}\label{alg:cdsac}
\begin{algorithmic}[1]
%\Function{DSAC2}{}
\State Initialize parameters $\theta$, $\bar{\theta}$, $\phi$
\State Initialize learning rates $\beta_{Z}$, $\beta_{\pi}$
\State $\B\leftarrow \emptyset$
% \State $k\leftarrow 0$
\Repeat
    \State Receive initial observation state $s_0$ 
    \For {\textbf{each} environment step}
        \State Select action $a_t \sim \pi_{\phi}(\cdot\given s_t)$
        \State Observe transition $s_{t+1} \sim P(\cdot \mid s_t,a_t)$
        \State Observe reward $r_t$ 
        \State $\B \leftarrow  \B \cup \{(s_t,a_t,r_t,s_{t+1})\}$  
    \EndFor
    \For {\textbf{each} gradient step}
        \State Uniformly sample a mini-batch of $N$ transitions $(s_i, a_i, r_i, s_{i+1})$ 
        \State $\theta \leftarrow \theta + \beta_{Z} \hat{\nabla}_{\theta}J_{Z}(\theta)$, (Eq. \eqref{cdsac_policyevaluation_objective}) \label{line:theta}
        \State $\phi \leftarrow \phi + \beta_{\pi} \hat{\nabla}_{\phi} J_{\pi}(\phi)$, (Eq. \eqref{actor_gradient})
        \State \text{Update target network}: $\bar{\theta} \leftarrow \tau \theta + (1 - \tau)\bar{\theta}$
    \EndFor
\Until{Stopping criterion} \label{line:epend}
\State \Return{$Z_{\theta}, \pi_{\phi}$}

\end{algorithmic}
\end{algorithm}  

%----------------------------------------------------------------

%\subsection{Gradient Stability}
\subsection{Confidence-Driven Value Update}
\label{sec:stability}
Fix a target distribution $\varphi_{Q',\sigma'}$, where $Q'$ and $\sigma'$  represent the Gaussian parameters.
The cost function $C(Q, \sigma)$ measures the energy distance between a candidate distribution and this target
$$C(Q,\sigma):=d_e(\varphi_{Q,\sigma},\varphi_{Q',\sigma'})= \int_{- \infty}^{\infty}{\left(F_{Q, \sigma}(x) - F_{Q',\sigma'}(x)\right)^2} dx.
$$

\begin{lemma}
\label{lem:1}
Let   $\varphi_{Q,\sigma}$ and $\varphi_{Q',\sigma'}$ be a current and a target distribution, respectively. Then
\begin{equation}
    \frac{\partial}{\partial  Q} C(Q,\sigma) =
    -\frac{2}{\sigma }
    B(Q,\sigma),
%    -\frac{1}{\sigma } \cdot 
%    \frac{B(Q,\sigma)}{C(Q,\sigma)},
\end{equation}
where
\begin{equation}
\label{eq:Bdef}
B(Q,\sigma):= \int_{- \infty}^{\infty}
{ (F_{Q, \sigma}(x) -F_{Q',\sigma'}(x)) \varphi_{Q,\sigma}(x) dx}.
\end{equation}
\end{lemma}
\begin{proof}
Let $f(x)=\varphi_{(0,1)}$ be the standard normal density function and $\Phi=F_{0,1}$ the standard normal cumulative distribution function. Applying the chain rule yields
$$
\frac{\partial}{\partial  Q}F_{Q,\sigma}
=\frac{\partial}{\partial  Q} \Phi\left(\frac{x-Q}{\sigma}\right)
=-\frac{1}{\sigma}f\left(\frac{x-Q}{\sigma}\right)=-\frac{\varphi_{Q,\sigma}}{\sigma}.
$$
Hence,
\begin{eqnarray}
\frac{\partial}{\partial  Q} C(Q,\sigma) &=& 
%\frac{1}{2C(Q,\sigma)}   
\int_{- \infty}^{\infty}
{2 (F_{Q, \sigma}(x) - F_{Q',\sigma'}(x)) \frac{\partial}{\partial  Q} F_{Q, \sigma}(x)  dx}\nonumber \\
&=& 
%\frac{1}{2C(Q,\sigma)}   
\int_{- \infty}^{\infty}
{-2 (F_{Q, \sigma}(x) - F_{Q',\sigma'}(x)) \frac{\varphi_{Q,\sigma}(x)}{\sigma} dx}\nonumber\\
 & =& 
 %\frac{-1}{\sigma C(Q,\sigma)}  
 -\frac{2}{\sigma }  
 \int_{- \infty}^{\infty}
{ (F_{Q, \sigma}(x) -F_{Q',\sigma'}(x)) \varphi_{Q,\sigma}(x) dx}.\nonumber
\end{eqnarray}
\end{proof}

\bigskip

\begin{figure}
    \hspace{-2cm}
  \includegraphics[scale=0.25]{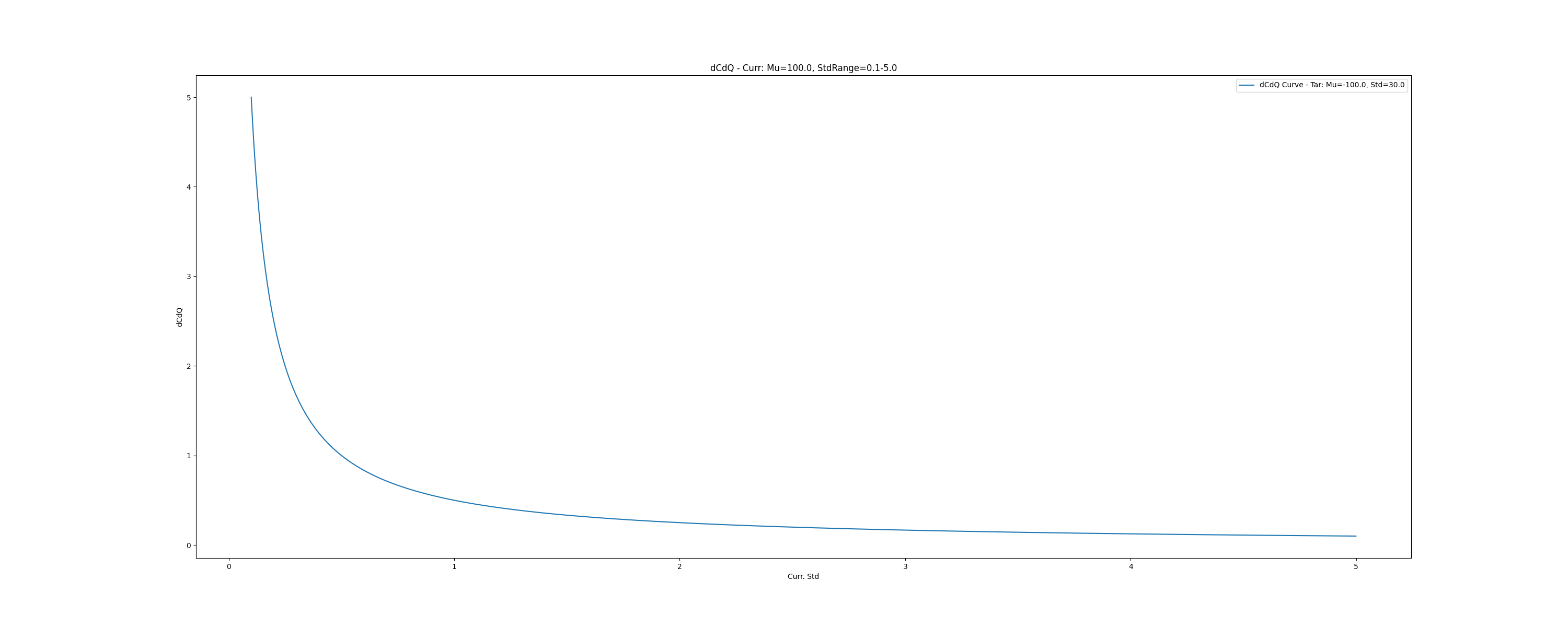}
    \caption{Value of $\frac{\partial}{\partial  Q} C(Q,\sigma)$ for varying $\sigma\in [\lb\sigma,\ub\sigma]$}
    \label{fig:C_Q}
\end{figure}
The following Lemma and Proposition establish the core result of the analysis. 
\begin{lemma}
\label{lem:2}
Let   $\varphi_{Q,\sigma}$ and $\varphi_{Q',\sigma'}$ be
a current 
and a target distribution, respectively. Then
$$
\lim_{\sigma\to\infty }\frac{\partial}{\partial  Q} C(Q,\sigma)=0.
$$
\end{lemma}
\begin{proof}
Since $F_{Q, \sigma}(x)\in [0,1]$ and $\|\varphi_{Q, \sigma}\|_{1}=1$, it holds from the H\"older inequality
$$
|B(Q,\sigma)|\leq 
\|(F_{Q, \sigma} -F_{Q',\sigma'}) \cdot \varphi_{Q,\sigma}\|_{1}
\leq \|F_{Q, \sigma} -F_{Q',\sigma'}\|_{\infty} \cdot  \|\varphi_{Q,\sigma}\|_{1} \leq 1.
$$
This gives
$$
\lim_{\sigma\to\infty } 
\left|\frac{\partial}{\partial  Q} C(Q,\sigma)\right| = 
\lim_{\sigma\to\infty } 
\frac{2}{\sigma } \cdot 
|B(Q,\sigma)|
\leq 
\lim_{\sigma\to\infty } 
\frac{2}{\sigma 
} =0.
$$
\end{proof}

\begin{proposition}
Denote  by $\Psi_\theta(s_t, a_t):=\frac{\partial}{\partial Q_{\theta_1}(s_t, a_t)} d_e(\hat{\T}^{\pi}_H Z(s_t, a_t), Z_{\theta}(s_t, a_t))$ a gradient weight of Eq. \eqref{cdsac_policyevaluation_objective} such that\\
$\hat{\nabla}_{\theta_1} J_Z(\theta)=\frac{1}{|\hat{\B}|} \sum_{\hat{\B}}
[\Psi_\theta(s_t, a_t)\nabla_{\theta_1} Q_{\theta_1}(s_t, a_t)]$.
Then
\begin{equation}
\label{eq:grad_J_Z}
\lim_{\sigma_{\theta_2}(s_t, a_t)\to \infty}\Psi_\theta(s_t, a_t)=0.
%+\Phi(s, a)\nabla_{\theta_2} \sigma_{\theta_2}(s, a
\end{equation}
\end{proposition}   
\begin{proof}
From Lemma \ref{lem:1} it follows
$$
\Psi_\theta(s_t, a_t)=\frac{\partial}{\partial Q_{\theta_1}(s_t, a_t)} d_e(\hat{\T}_H^\pi Z(s_t, a_t), Z_{\theta}(s_t, a_t)) 
=-
\frac{2B(Q_{\theta_1}(s_t, a_t),\sigma_{\theta_2}(s_t, a_t))}{\sigma_{\theta_2}(s_t, a_t)}.
$$
The statement follows from Lemma
\ref{lem:2}.
\end{proof}
Figure \ref{fig:C_Q} shows 
that the gradient weight $\Psi_\theta(s_t, a_t)$ is quickly reducing, if the variance $\sigma_{\theta_2}(s_t,a_t)$ is increased.
The analysis reveals that for a particular pair $(s_t,a_t)$, the magnitude of the $Q_{\theta_1}$ update is reversely scaled by the variance $\sigma_{\theta_2}$, such that the update vanishes as $\sigma_{\theta_2} \rightarrow \infty$.   
Hence,
$Q_{\theta_1}(s_t,a_t)$ is updated more aggressively at state-action pairs $(s_t,a_t)$ where $\sigma_{\theta_2}(s_t,a_t)$ is small, reflecting temporal stability of the target.  This update type is therefore called {\it confidence-driven}. 
%#######
\subsection{Effect on Overestimation Bias}
\label{sec:overestimation}
A pervasive challenge in $Q$-based reinforcement learning is the overestimation bias that arises from function approximation in noisy environments \citep{Thrun1993IssuesIU}.
In this section, it is shown that the influence of the overestimation error at $(s_t,a_t)$ on the gradient of the energy loss function \eqref{cdsac_policyevaluation_objective} is small, if the variance $\sigma_{\theta_2} (s_t,a_t)$ of $Z_\theta(s_t,a_t)$ is large.
Denote by $\tilde Q$ and $Q^*$ a noisy  and an exact mean (Q-value), respectively. 
Consider a noisy target 
$$\tilde Y(s_t, a_t):= r(s_t, a_t) + \gamma Z(s_{t+1},a_{t+1}), \quad a_{t+1}=\argmax_{a\in {\cal A}}\tilde Q(s_{t+1}, a) $$ 
and an exact target
$$Y^*(s_t, a_t):= r(s_t, a_t) + \gamma Z(s_{t+1},a_{t+1}), \quad a_{t+1}=\argmax_{a\in {\cal A}} Q^*(s_{t+1}, a). $$

\begin{proposition}
Let
$\Delta \Psi_\theta(s_t, a_t):=|\tilde \Psi(s_t,a_t)-\Psi^*(s_t,a_t)|$ be the gradient weight error of the energy distance \eqref{cdsac_policyevaluation_objective} regarding  noisy and exact gradient weights, where 
$$
\tilde \Psi(s_t,a_t):=\frac{\partial}{\partial Q_{\theta_1}} d_c(Z_{\theta}(s_t,a_t), \tilde Y(s_t,a_t)),\quad
\Psi^*(s_t,a_t):=\frac{\partial}{\partial Q_{\theta_1}} d_c(Z_{\theta}(s_t,a_t), Y^* (s_t,a_t)).
$$ 
 Then 
\begin{equation}
    \label{eq:CramerOverestimation}
    \lim_{\sigma_{\theta_2}(s_t, a_t)\to \infty}\Delta \Psi_\theta(s_t, a_t)=0.
\end{equation}
\end{proposition}   
\begin{proof}
Define $B^*$ and $\tilde{B}$
via \eqref{eq:Bdef} using $Y^*$ and $\tilde{Y}$, respectively.
. From Lemma \ref{lem:1} it follows
$$
\left|\tilde \Psi(s_t,a_t)-\Psi^*(s_t,a_t)\right| = 2\left|
     \frac{\tilde B_\theta(s_t,a_t)}{\sigma_{\theta_2}(s_t,a_t)} 
     %\frac{
     %\tilde B_\theta(s_t,a_t)
     %}{\tilde C_\theta(s_t,a_t)}
    -  \frac{B^*_\theta(s_t,a_t)}{\sigma_{\theta_2}(s_t,a_t) } 
    %\frac{B^*_\theta(s_t,a_t)}{C^*_\theta(s_t,a_t)}
\right|
\leq \frac{2}{\sigma_{\theta_2}(s_t,a_t) },
$$
since $|B^*_\theta(s_t,a_t)|\leq 1$
and $|\tilde B_\theta(s_t,a_t)|\leq 1$.
This proves \eqref{eq:CramerOverestimation}.
\end{proof}
Figure \ref{fig:over_est} shows that the gradient weight error $\Delta \Psi_\theta(s_t, a_t)$ is quickly reducing  
if the variance $\sigma_{\theta_2}(s_t,a_t)$ is increased. 
The variance is elevated at state-action pairs exhibiting high return stochasticity, coinciding with the regions most susceptible to value overestimation \cite{Lan2020Maxmin}.

\begin{figure}[h]
 \hspace{-2cm}
    %\centering
  \includegraphics[scale=0.25]{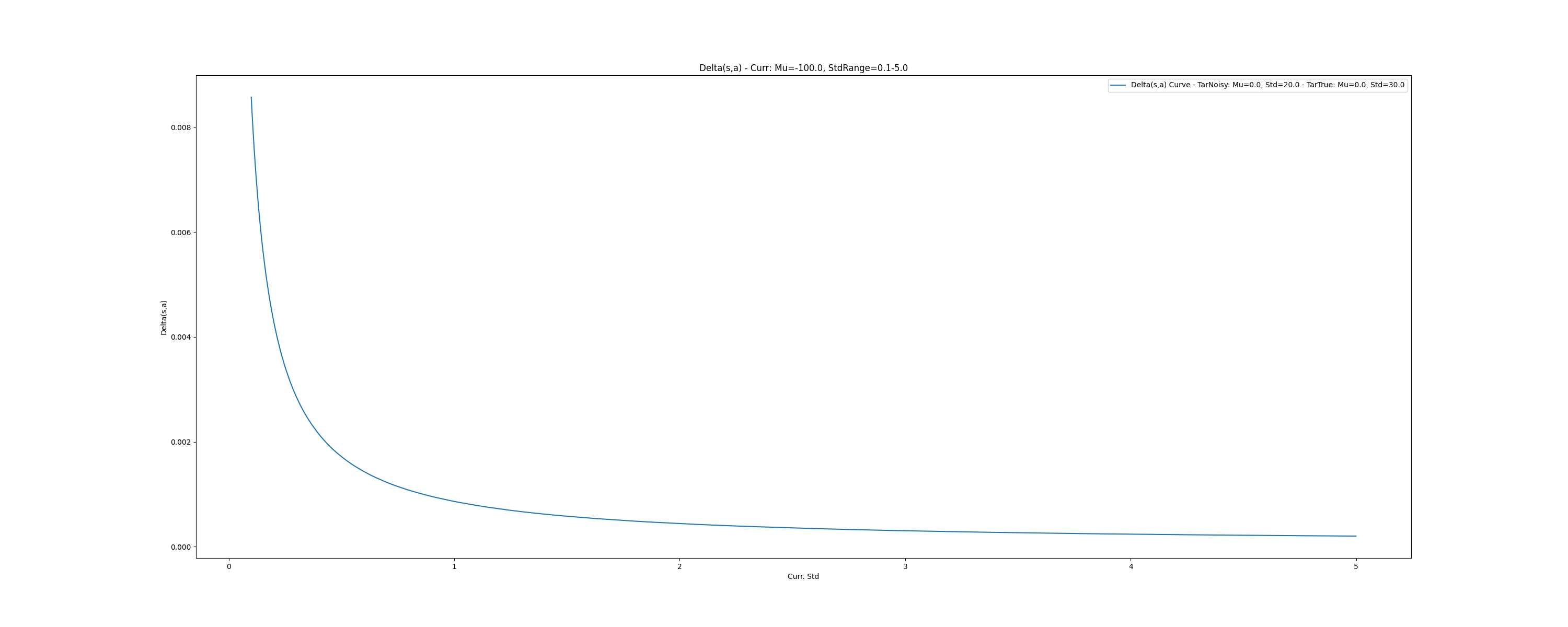}
    \caption{Value of $\Delta \Psi$ for varying $\sigma\in [\lb\sigma,\ub\sigma]$}
    \label{fig:over_est}
\end{figure}
%
%\att{Overestimation Experiment}
%---------------------------------------------------------------------------------

%\subsection{Algorithm}
%\label{sec:Algorithm}

%----------------------------------------------------------------

\section{Numerical Experiments}
\label{sec:num}
The experiments are conducted with the same hyperparameter values as in \cite{Haarnoja2018SoftAO}, except for C-DSAC-specific ones (e.g, limits on standard deviation, ... etc. ), see Table \ref{tab:cdsac_hyper}. Implementation-level nuances and low-level optimizations exert a substantial influence on an algorithm's empirical performance \citep{Engstrom2020ImplementationMatters}. It is assumed that the original implementation used for the baseline results in \cite{Haarnoja2018SoftAO} did not employ additional optimization techniques such as value function clipping, orthogonal initialization, layer scaling, learning rate annealing, reward/observation clipping, or observation and reward normalization. 
To be consistent with these assumptions, these code-level optimization techniques is not employed in the C-DSAC implementation. 
Using the same network architecture as the SAC experiments would grant C-DSAC a disproportionate advantage, owing to its dual-output critic. 
Therefore, the last hidden layer  is reduced by one neuron. Consequently, the C-DSAC critic has slightly fewer parameters than those reported in the original SAC configuration, see Appendix \ref{sec:hyperparameters}. 
Although \cite{Ma2020DistributionalSA} and \cite{Duan2020DistributionalSA} claim superior performance over SAC baselines, it is important to note that the former uses more complex neural network architectures with batch normalization for the value approximator.

\subsection{Comparative Evaluation}
The goal of the experiments is to evaluate the impact of the energy distance on the entropy-augmented actor-critic framework across diverse and complex environments \cite{brockman2016openai}, see Figure \ref{fig:gym_envs}: Hopper-v4, $(\s \times \A) \in \mathbb{R}^{11} \times \mathbb{R}^3$; Ant-v4 $(\s \times \A) \in \mathbb{R}^{111} \times \mathbb{R}^8$; Humanoid-v4 $(\s \times \A) \in \mathbb{R}^{376} \times \mathbb{R}^{17}$; HalfCheetah-v4 $(\s \times \A) \in \mathbb{R}^{17} \times \mathbb{R}^6$; Walker2d-v4 $(\s \times \mathcal{A}) \in \mathbb{R}^{17} \times \mathbb{R}^6$. 
The entropy coefficient is constant for each environment, following the configuration in \cite{Haarnoja2018SoftAO}. Environment-specific values for $\alpha$ are given in Table \ref{tab:env_entropy}.
Following the recommendations of \cite{Henderson2017DeepRL}, evaluation rollouts are performed every $10^3$ iterations . During evaluation, agents follow a deterministic policy derived from the distribution mean.
The agent is trained for one million iterations in each environment. The results are averaged over five random seeds. 
The graphs in Figure \ref{fig:eval_CDSAC} showcase (among others) the evaluation performances of C-DSAC and SAC in the benchmark environments.  
%\newline
%
Especially in the most difficult environment, Humanoid-v4, C-DSAC clearly outperforms SAC and reaches an average reward of about 5715, a value that is achieved by SAC only after approximately four million iterations. 
Notably, C-DSAC performance has not yet reached a plateau by the conclusion of training, suggesting further potential for improvement with extended computation.
Although C-DSAC and SAC reach the same average rollout in the Hopper-v4 environment, C-DSAC does so considerably faster. 
The instability after $4 \times 10^5 $ iterations can be due to the occurrence of catastrophic forgetting in one of the runs. However, it manages to recover. 
In the Walker2d-v4 environment, C-DSAC attains higher mean rewards, but the training profile reveals greater stochasticity relative to the baseline.
The least stable training occurs in the Ant-v4 environment, where high stochasticity can be observed. Still, C-DSAC outperforms SAC in Ant-v4.
In HalfCheetah-v4, C-DSAC obtains an average reward of less than $10,000$ in the end of training and is therefore outperformed by SAC in this environment. 

To provide more details about the performance, Table \ref{tab:best_performance} presents the best model performances across all environments, reporting average episodic rewards over 100 runs.

\begin{figure}[ht]
    \centering
    \begin{subfigure}{0.22\textwidth}
        \centering
        \includegraphics[scale=0.5]{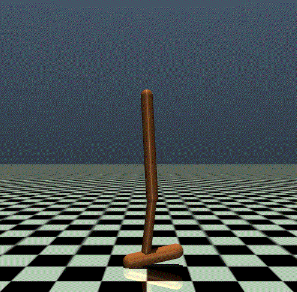}
        \label{fig:Hopper-v4}
        \caption{Hopper-v4}
    \end{subfigure}
    \begin{subfigure}{0.22 \textwidth}
        \centering
        \includegraphics[scale=0.49]{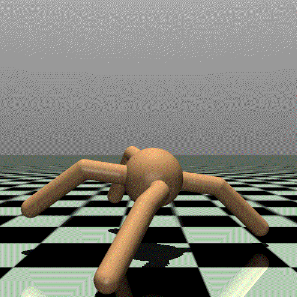}
        \label{fig:Ant-v4}
        \caption{Ant-v4}
    \end{subfigure}
    \begin{subfigure}{0.22\textwidth}
        \centering
        \includegraphics[scale=0.5]{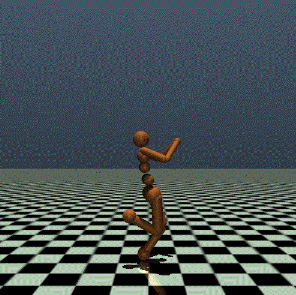}
        \label{fig:Humanoid-v4}
        \caption{Humanoid-v4}
    \end{subfigure}
    \begin{subfigure}{0.22\textwidth}
        \centering
        \includegraphics[scale=0.5]{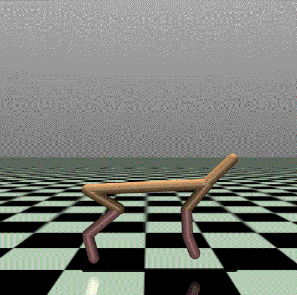}
        \label{fig:HalfCheetah-v4}
        \caption{HalfCheetah-v4}
    \end{subfigure}
    \begin{subfigure}{0.17\textwidth}
        \centering
        \includegraphics[scale=0.5]{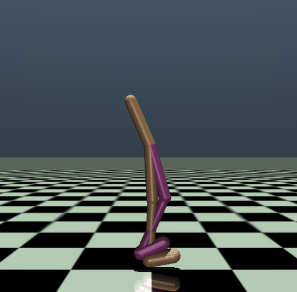}
        \label{fig:Walker2d-v4}
        \caption{Walker2d-v4}
    \end{subfigure} 
    \caption{Testing environments}
    \label{fig:gym_envs}
\end{figure}

% comparison CDSAC - SAC
\begin{figure}[htbp]
\vspace{-1cm}
    \centering

    % First row
    \begin{subfigure}[b]{\textwidth}
        \centering
        % \begin{minipage}{0.45\textwidth}
         \begin{minipage}{0.50\textwidth}
            \centering
            %\vspace{-0.5cm}
            \includegraphics[width=\textwidth]{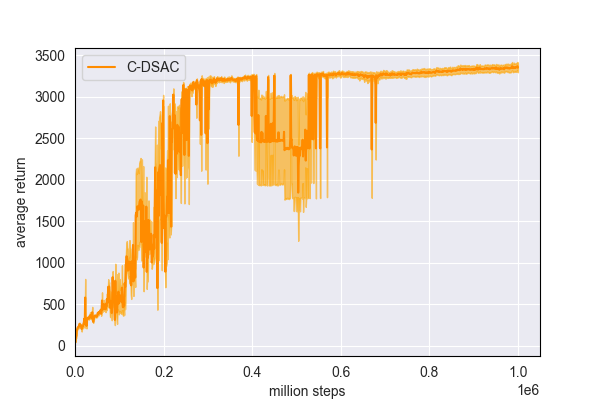}
        \end{minipage}
        \hfill
        \begin{minipage}{0.45\textwidth}
            \centering
            \vspace{0.5cm}
            \includegraphics[width=\textwidth]{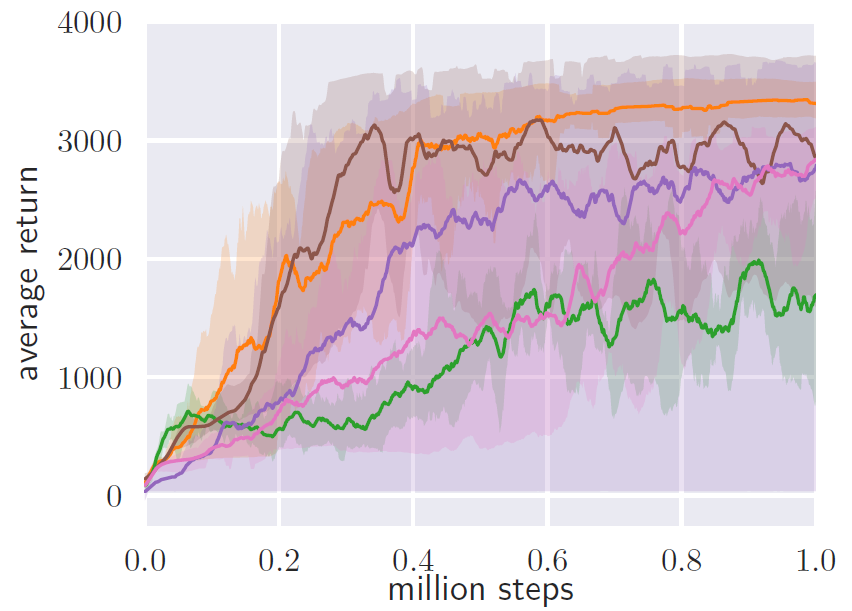}
        \end{minipage}
        \caption{Hopper}
    \end{subfigure}

    % Second row
    \begin{subfigure}[b]{\textwidth}
        \centering
        \begin{minipage}{0.52\textwidth}
            \centering
            \includegraphics[width=\textwidth]{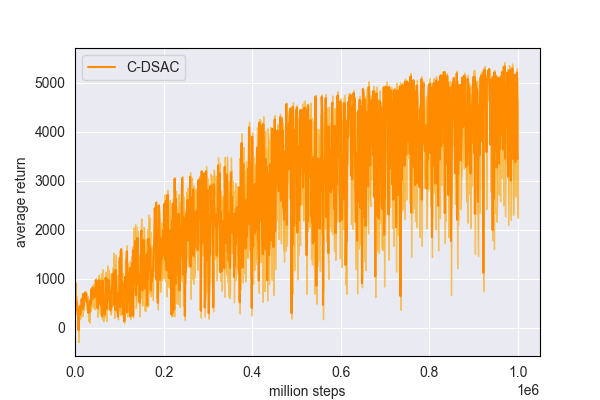}
        \end{minipage}
        \hfill
        \begin{minipage}{0.45\textwidth}
            \centering
            \vspace{0.5cm}
            \includegraphics[width=\textwidth]{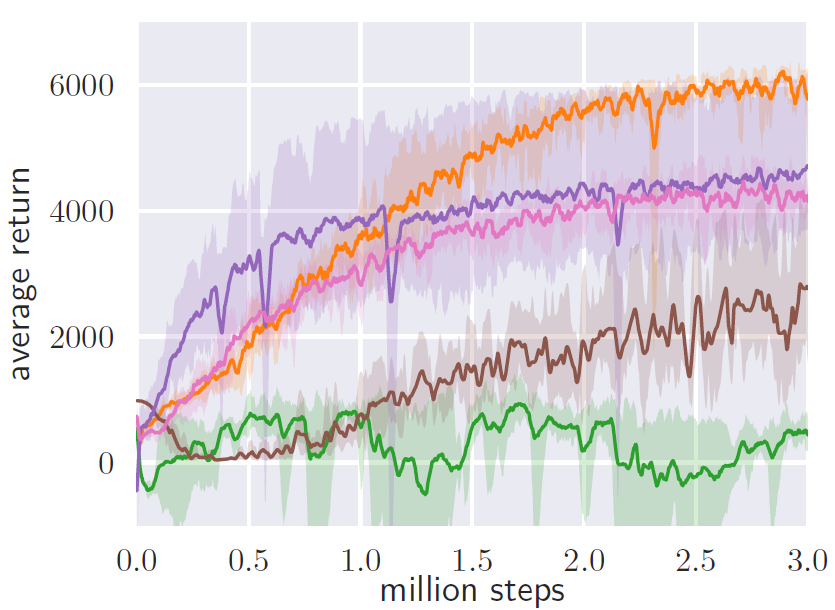}
        \end{minipage}
        \caption{Ant (left one million, right three million steps)}
    \end{subfigure}

    % Third row
    \begin{subfigure}[b]{\textwidth}
        \centering
        \begin{minipage}{0.52\textwidth}
            \centering
            \includegraphics[width=\textwidth]{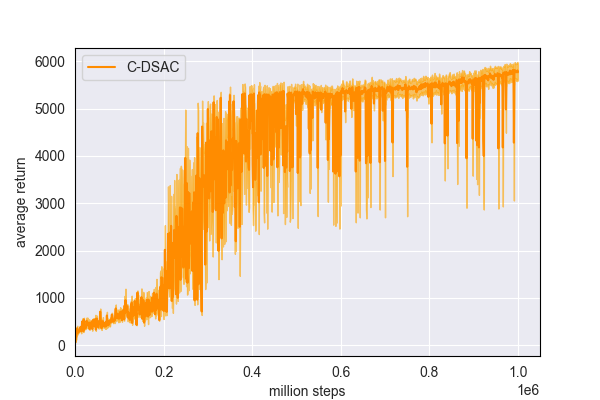}
        \end{minipage}
        \hfill
        \begin{minipage}{0.45\textwidth}
            \centering
            \vspace{0.5cm}
            \includegraphics[width=\textwidth]{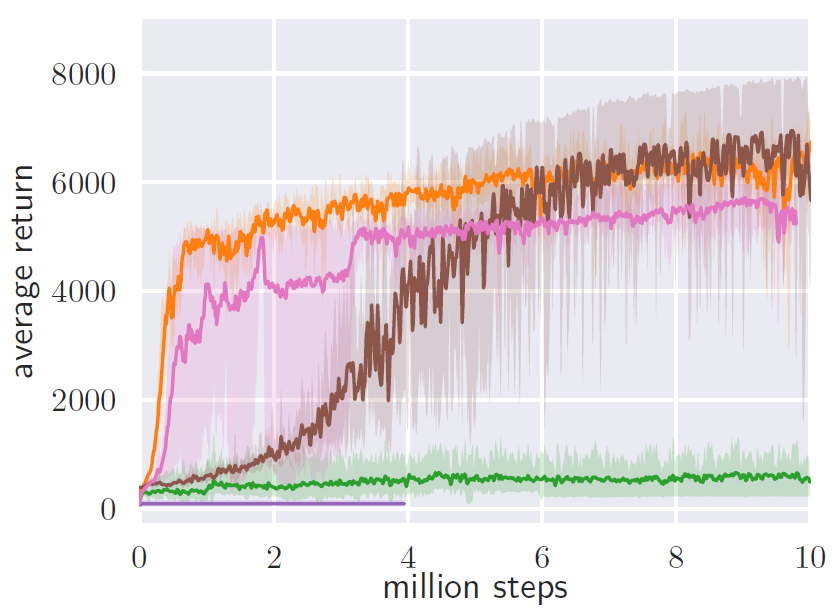}
        \end{minipage}
        \caption{Humanoid (left one million, right ten million steps)}
    \end{subfigure}

    % Fourth row
    \begin{subfigure}[b]{\textwidth}
        \centering
        \begin{minipage}{0.52\textwidth}
            \centering
            \includegraphics[width=\textwidth]{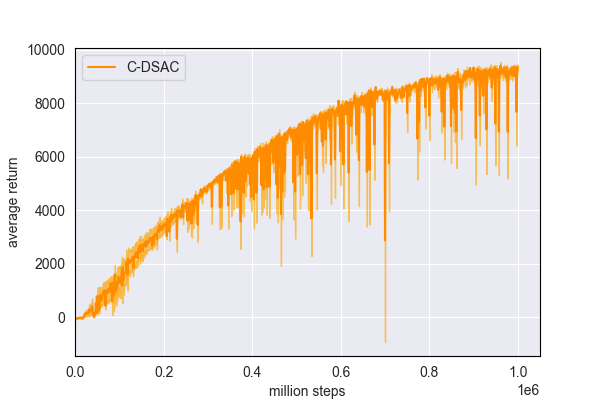}
        \end{minipage}
        \hfill
        \begin{minipage}{0.45\textwidth}
            \centering
            \vspace{0.6cm}
            \includegraphics[width=\textwidth]{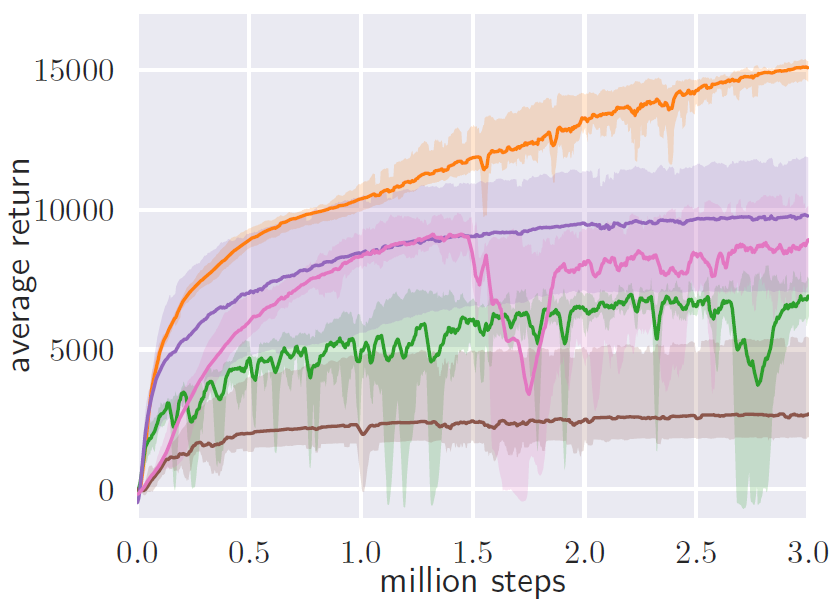}
        \end{minipage}
        \caption{HalfCheetah (left one million, right three million steps)}
    \end{subfigure}
    %\caption{Main caption for the entire figure}
\end{figure}

% Continue Comparison Figure
\begin{figure}[htbp]
    \ContinuedFloat
    \centering
    % Fifth row
    \begin{subfigure}[b]{\textwidth}
        \centering
        \begin{minipage}{0.52\textwidth}
            \centering
            \includegraphics[width=\textwidth]{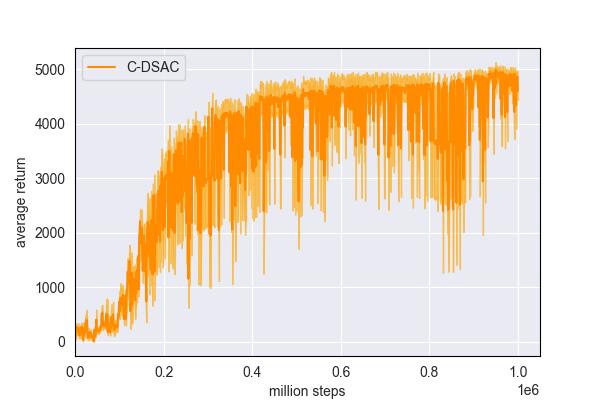}
        \end{minipage}
        \hfill
        \begin{minipage}{0.45\textwidth}
            \centering
            \vspace{0.6cm}
            \includegraphics[width=\textwidth]{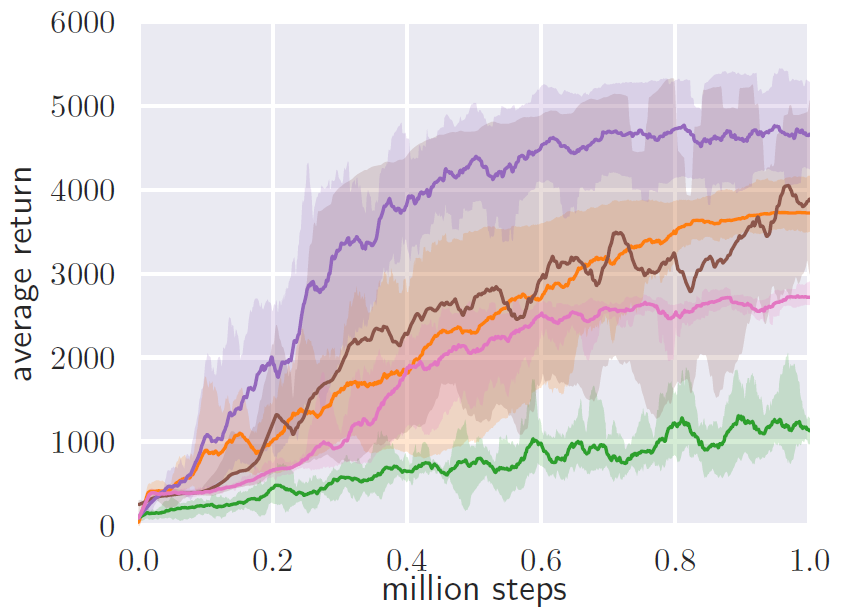}
        \end{minipage}
        \caption{Walker}
    \end{subfigure}
    
    \caption{Orange curves represent C-DSAC (left) and SAC (right, from \citep{Haarnoja2018SoftAO}). Other curves show performances of DDPG (green), PPO (brown), SQL (pink), and TD3 (violet).}
    \label{fig:eval_CDSAC}
\end{figure}

% % Insert C-DSAC reward rollout curves here old
% \begin{figure}[ht]
%     \centering
%     \begin{subfigure}{0.32\textwidth}
%         \centering
%         \includegraphics[scale=0.4]{Images/Hopper-v4_Rollout_Averaged_Scaled.png}
%         \label{fig:Hopper-v4}
%         \caption{Hopper-v4}
%     \end{subfigure}
%     \begin{subfigure}{0.32 \textwidth}
%         \centering
%         \includegraphics[scale=0.20]{Images/Walker2d-v4_Rollout_Averaged.png}
%         \label{fig:Ant-v4}
%         \caption{Walker-v4}
%     \end{subfigure}
%     \begin{subfigure}{0.32\textwidth}
%         \centering
%         \includegraphics[scale=0.20]{Images/HalfCheetah-v4_Rollout_Averaged.png}
%         \label{fig:Humanoid-v4}
%         \caption{HalfCheetah-v4}
%     \end{subfigure}
%  \\   
%     \begin{subfigure}{0.4\textwidth}
%         \centering
%         \includegraphics[scale=0.2]{Images/Ant-v4_Rollout_Averaged.png}
%         \label{fig:HalfCheetah-v4}
%         \caption{Ant-v4}
%     \end{subfigure}
%     \begin{subfigure}{0.4\textwidth}
%         \centering
%         \includegraphics[scale=0.2]{Images/Humanoid-v4_Rollout_Averaged.png}
%         \label{fig:Walker2d-v4}
%         \caption{Humanoid-v4}
%     \end{subfigure}  
%     \caption{C-DSAC Baselines}
%     \label{fig:C-DSAC_Baselines}
% \end{figure}

% \begin{figure}[ht!]
%  % \centering
%  \hspace{-1.8cm}
%  \includegraphics[scale=0.80]{Images/SAC_Baselines.png}
%  \caption{SAC Baselines}
%  \label{fig:SAC_Baselines}
% \end{figure}

\begin{table}[ht]
    \centering
 \caption{Average performance over 100 runs of the best models in the training runs}
    \begin{tabular}{>{\centering\arraybackslash}m{5cm}| >{\centering\arraybackslash}m{5cm}}
        \hline
        \textbf{Environment} & \textbf{Best Performance} \\
        \hline
        Hopper-v4 & $3352 \pm 97$ \\
        \hline
        \multirow{1}{=}{\centering Ant-4} & $5376 \pm 79$\\ 
        \hline
        Humanoid-v4 & $5715 \pm 255$\\
        \hline
        \multirow{1}{=}{\centering HalfCheetah-4} & $10023 \pm 228$\\ 
        \hline
        Walker-v4 & $4808 \pm 223$ \\
        \hline
    \end{tabular}
    \label{tab:best_performance}
\end{table}

\subsection{Comparative Implementation Evaluation against a Standard Baseline}
To assess the quality of the C-DSAC implementation - and therefore its effect on the performance, its code is modified to revert it to SAC by adjusting the loss function and critic network architecture. The experiments are repeated on HalfCheetah-v4 with five runs, Figure \ref{fig:SAC_Implementation}. The SAC hyperparameter are set to values given in  Appendix \ref{sec:hyperparameters}.  The custom SAC implementation based on the C-DSAC code achieved an average reward of slightly above 3000, compared to over 10000 in the original SAC implementation after one million iterations. The observed performance gap potentially stems from unoptimized implementation details, suggesting that C-DSAC has yet to reach its theoretical performance ceiling.

\begin{figure}[H]
 \centering
 \includegraphics[scale=0.40]{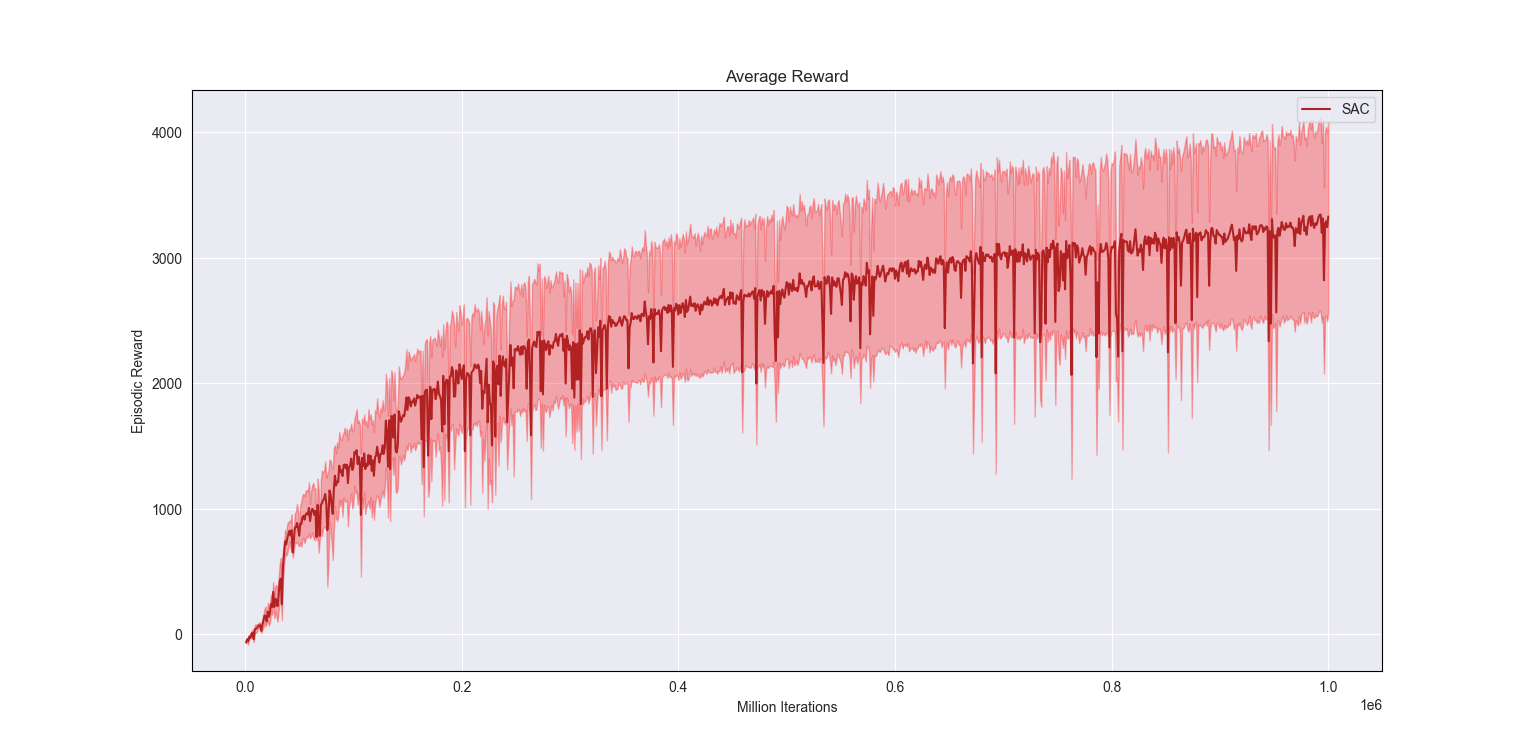}
 \caption{Performance of SAC based on C-DSAC implementation on HalfCheetah-v4}
 \label{fig:SAC_Implementation}
\end{figure}
%----------------------------------------------------------------

\section{Conclusion}
\label{sec:concl}
This work presented C-DSAC, a distributional maximum-entropy deep 
reinforcement learning algorithm. The objective function incorporates an energy distance-based metric to estimate the expected Q-values and their associated variances.
It has been formally established that distributional soft policy iteration, utilizing the energy distance metric for evaluation alongside standard policy improvement, converges to an optimal policy.
Numerical experiments show that C-DSAC outperforms SAC in most benchmark environments despite a sub-optimal implementation. The performance seems to be superior especially for complex environments, 
like the Humanoid-v4 environment.  
The presented analysis suggests that the superior performance of C-DSAC  may among others be attributed to a confidence-driven update of Q-values. It was shown that gradient weights of the critic loss function are small at state-action pairs with high variance. This results in a more conservative update for high-variance targets, reducing adaptation to potentially overestimated values. In this context, high variance indicates low confidence in the target, which improves only after repeated observations. 
As a consequence,  C-DSAC learns Q-values in a two-phase process: First, building confidence in the target, then adapting to its expected value. 
While this confidence-building phase may imply slow and unstable learning, the numerical results indicate that the reduction in the effect of overestimation outweighs the additional cost.
The C-DSAC framework and the current implementation remain open to various optimizations.  Experiments with different initialization values can be carried out, particularly for the critic's variance head. For a fair comparison with SAC baselines, automatic temperature adjustment was omitted, 
but could be integrated to enhance performance. Additionally, an exploration strategy based on the value distribution can be developed.
C-DSAC models the return with a fixed-moment distribution. 
Efforts can be made to efficiently apply GMMs with complex neural network approximators to allow for a more expressive return distribution.
Another promising direction for research is the utilization of task and reward decomposition based on the moments of the value distribution. 

\section*{Acknowledgements}
This work has been funded by  grant LFF FV 90 of the State Research Funding Hamburg (Landesforschungsf\"orderung Hamburg), and by PID2021-123278OB-I00 from the Spanish Ministry of Science and Innovation financed by “ERDF A way of making Europe".


\vskip 0.2in
%\bibliography{RL}

\appendix

% \section{Understanding DSAC}
% \begin{lema}
% (Distributional Soft Policy Evaluation with $\sigma$-fitting strategy    
% \end{lema}

% \begin{lema}
% (Soft Policy Improvement)
% \end{lema}

\section{Standard Reinforcement Learning}
The proofs provided below rely on established results \citep{Sutton1998intro, haarnoja2019autotemp}. They are included to ensure the self-containment of this work and to clarify the foundations upon which our Theorems and Lemmas are built.
\subsection{Policy Evaluation} 
\label{sec:policyevaluation_rl}
\begin{lemma}
\label{lem:policyevaluation_rl} 
Let $Q_1(s,a)$ and $Q_2(s,a)$ be two bounded state-action value functions. Assume $|\A| < \infty$, $|\s| < \infty$ and $0 \leq \gamma < 1$.   
Then $\T^{\pi}$ according to \eqref{eq:BO2} is a $\gamma$-contraction, i.e.
\begin{equation}
    \label{eq:contraction2}
    \|\T^{\pi}Q_1 - \T^{\pi}Q_2\|_{\infty} \leq \gamma \|Q_1 - Q_2\|_{\infty}.
\end{equation}
\end{lemma}
\begin{proof}
\label{eq:eval_rl_ac}
To show that \eqref{eq:contraction2} holds, the left-hand side is expanded. Set $\E_{P} [\cdot] := \E_{s_{t+1} \sim P(\cdot \mid s_t, a_t)}[\cdot]$ and $\E_{a_{t} \sim \pi(\cdot \mid s_{t})}[Q(s_t, a_t)] := Q(s_{t}, \pi(\cdot \mid s_t))$, then
\begin{align*}
&\|\T^{\pi}Q_1 - \T^{\pi}Q_2\|_{\infty} \\
&=\|\left(r(s_t,a_t) + \gamma \E_{P}[ Q_1(s_{t+1}, \pi(\cdot \mid s_{t+1}))]\right) - \left(r(s_t,a_t) + \gamma\E_{P}[ Q_2(s_{t+1}, \pi(\cdot \mid s_{t+1}))] \right)\|_{\infty} \\
&= \sup_{s_t \in \s,a_t \in \A}  |\gamma\E_{P}[ Q_1(s_{t+1}, \pi(\cdot \mid s_{t+1})) ]- \gamma \E_{P}[Q_2(s_{t+1}, \pi(\cdot \mid s_{t+1}))]|\\
&= \gamma \sup_{s_t \in \s,a_t \in \A}| \E_{P}[(Q_1(s_{t+1}, \pi(\cdot \mid s_{t+1})) - Q_2(s_{t+1}, \pi(\cdot \mid s_{t+1})))]|\\
&\leq \gamma \sup_{s_t \in \s,a_t \in \A} \E_{P}[|Q_1(s_{t+1}, \pi(\cdot \mid s_{t+1})) - Q_2(s_{t+1}, \pi(\cdot \mid s_{t+1}))| ]\\
&\leq \gamma \sup_{s \in \s} |Q_1(s, \pi(\cdot \mid s)) - Q_2(s, \pi(\cdot\mid s))| \\
& \leq \gamma \sup_{s \in \s, a \in \A} |Q_1(s,a) - Q_2(s,a)| = \gamma \|Q_1 - Q_2\|_{\infty}. \\
\end{align*}
Since $\T^{\pi}$  is a $\gamma$-contraction,
it follows from the  Banach theorem that  
the series $\{Q_k\}$ with $Q_{k+1} := \T^\pi Q_k$
converges to  a unique fixed point $Q^{\pi}$, i.e. $Q^\pi = \T^\pi Q^\pi$ and $\|Q_k - Q^{\pi}\|_{\infty}\rightarrow 0$ as $k \rightarrow \infty$.
\end{proof}

\subsection{Policy Improvement}
\label{sec:policyimprovement_rl}

\begin{lemma}
\label{lem:policyimprovement_rl}
 Given the policy update rule in 
\eqref{eq:policyupdate_rl} and assuming bounded state-action value functions, $|\A| < \infty$, $0 \leq \gamma < 1$ then 

$\E_{a_t \sim \pi_k(\cdot \mid s_t)}[Q^{\pi_k}(s_t, a_t)] \leq \E_{a_t \sim \pi_{k+1}(\cdot \mid s_t)}[Q^{\pi_{k+1}}(s_t, a_t)] \quad \forall s_t \in \s$.
\end{lemma}

\begin{proof}
The proof relies on expanding the Q-function in the Bellman equation repeatedly with actions of the improved policy.  For convenience, the notation $\pi' := \pi_{k+1}$ and $\pi := \pi_k$ and $\E_{\pi', P, k}[\cdot \mid s_t] :=
\E_{a_{t:k-1}\sim \pi'(\cdot\mid s_{t:k-1}), s_{t+1:k} \sim P(\cdot \mid s_{t:k-1}, a_{t:k-1})}[\cdot \mid s_t]$ is used. Let $ Q^{\pi}(s_t, \pi(\cdot \mid s_t)) := \E_{a_t \sim \pi(\cdot \mid s_t)}[Q^{\pi}(s_t,a_t)]$, then

\begin{align*}
Q^{\pi}(s_t, \pi(\cdot \mid s_t)) &\leq Q^{\pi}(s_t, \pi'(\cdot \mid s_t)) \\
&=\E_{\pi', P, t+1}[r(s_t,a_t) + \gamma Q^{\pi}(s_{t+1}, \pi(\cdot \mid s_{t+1})) \mid s_t] \\
&\leq \E_{\pi', P, t+1}[r(s_t, a_t) + \gamma Q^{\pi}(s_{t+1},\pi'(\cdot \mid s_{t+1})) \mid s_t] \\
&=\E_{\pi', P, t+2}[r(s_t,a_t) + \gamma (r(s_{t+1}, a_{t+1}) + \gamma Q^{\pi}(s_{t+2}, \pi(\cdot \mid s_{t+2}))) \mid s_t] \\
&= \E_{\pi', P, t+2}[r(s_t, a_t) + \gamma r(s_{t+1}, a_{t+1}) + \gamma^2 Q^{\pi}(s_{t+2}, \pi(\cdot \mid s_{t+2})) \mid s_t] \\
& \leq \E_{\pi', P, t+3}[r(s_t, a_t) + \gamma r(s_{t+1}, a_{t+1}) + \gamma^2 r(s_{t+2}, a_{t+2}) + \gamma^3 Q^{\pi}(s_{t+3}, \pi(\cdot \mid s_{t+3})) \mid s_t] \\
&\vdots \\
& \leq \E_{\pi', P, T-1}[r(s_t, a_t) + \gamma r(s_{t+1}, a_{t+1}) + \gamma^2 r(s_{t+2}, a_{t+2}) + \hdots + \gamma^{T-1} r(s_{T-1}, a_{T-1}) \mid s_t] \\
&= Q^{\pi'}(s_t, \pi'(\cdot \mid s_t)).
\end{align*}
Thus for any $s_t$,
$$ \E_{a_t \sim \pi_k(\cdot \mid s_t)}[Q^{\pi_k}(s_t,a_t)] = Q^{\pi}(s_t,\pi(\cdot \mid s_t)) \leq Q^{\pi'}(s_t ,\pi'(\cdot \mid s_t)) = \E_{a_t \sim \pi_{k+1}(\cdot \mid s_t)}[Q^{\pi_{k+1}}(s_t,a_t)].$$ 
\end{proof}

%---------------

\subsection{Policy Iteration}
\label{sec:policyiteration_rl}
\begin{lemma}
\label{lem:policyiteration_rl}
Let the reward be bounded, $|\A|<\infty$, $|\s|<\infty$, and $0 \leq \gamma < 1$. Let $\Pi$ be a the set of all stationary (possibly stochastic) policies. In particular, $\Pi$ contains the point-mass policies. Alternating between exact policy evaluation and policy improvement from some initial policy $\pi_0 \in \Pi$ and state-action function $Q^{(0)}$, the process terminates after finitely many improvement steps at the optimal policy $\pi^*$ satisfying $Q^{\pi^*} = Q^*$. 
\end{lemma}
\begin{proof}
Fix $k$. By Lemma \ref{lem:policyevaluation_rl}, the inner evaluation sequence $\{Q_m\}$ for $Q_{m+1}:= \T^{\pi_k}Q_m$ approaches $Q^{\pi_k}$, i.e. $||Q_m - Q^{\pi_k}||_{\infty} \rightarrow 0$ as $m \rightarrow \infty$. It is established by Lemma \ref{lem:policyimprovement_rl} that an improvement of policy $\pi_{k}$ by Eq. \eqref{eq:policyupdate_rl}, yielding $\pi_{k+1}$, satisfies 
\\
$V^{\pi_k}(s_t) =\E_{a_t \sim \pi_k(\cdot \mid s_t)}[Q^{\pi_k}(s_t, a_t)] \leq \E_{a_t \sim \pi_{k+1}(\cdot \mid s_t)}[Q^{\pi_{k+1}}(s_t, a_t)] = V^{\pi_{k+1}}(s_t) \forall s_t \in \s$. 
Although policies are treated as stochastic mappings, the $\argmax$ in Eq. \eqref{eq:policyupdate_rl} implies that actions are taken deterministically, provided the action set is finite.
With finite states and actions and ties being broken deterministically per state (e.g. lexicographic smallest action), the set of greedy policies is also finite.
Under these circumstances, a greedy policy for a fixed $Q^{\pi_k}$ is unique, so $\pi_{k+1} \neq \pi_k$ implies strict increase in the value $V$ for at least one state.
Eventually $\pi_{k+1} = \pi_k$ in finite time, implying convergence to a fixed point $\pi_{k+1} = \pi_k = \pi^*$. 
The terminating policy $\pi^*$ is greedy w.r.t. its own $Q$, therefore satisfies the Bellman optimality equation and thus $Q^{\pi^*} = Q^*$. 
\end{proof}
%----------------------
%---------------------
\section{Maximum-Entropy Reinforcement Learning}
\label{sec:hrl_proofs}
\subsection{Soft Evaluation}
\begin{lemma}
\label{lem:policyevaluation_hr}
Let $Q_1(s,a)$ and $Q_2(s,a)$ be two bounded state-action value functions. Assume $|\mathcal{A}| < \infty$, $|\mathcal{S}| < \infty$, $0 \leq \gamma < 1$ and $\alpha \geq 0$ . Then the soft Bellman operator $\T_h^{\pi}$ according to \eqref{eq:BOH} is a $\gamma$-contraction
\begin{equation}
    \label{eq:contractH}
    \|\T_h^{\pi}Q_1 - \T_h^{\pi}Q_2\|_{\infty} \leq \gamma \|Q_1 - Q_2\|_{\infty}.
\end{equation}
\end{lemma}
\begin{proof}
By linearity of expectation, the entropy term can be separated from the Q-function. Define 
$$r_h(s_t, a_t) := r(s_t, a_t) + \gamma \E_{s_{t+1}\sim P(\cdot \mid s_t, a_t)}[\alpha \mathcal{H}(\pi(\cdot \mid s_{t+1}))].$$
Note that
$$\T_h^{\pi}Q(s_t, a_t) = r_h(s_t,a_t) + \gamma \E_{\substack{s_{t+1} \sim P(\cdot \mid s_t,a_t), \\ a_{t+1} \sim \pi(\cdot \mid s_{t+1})}}[Q(s_{t+1}, a_{t+1})]$$ and since entropy is bounded for discrete actions, $\T_h^{\pi}$ acts on a Banach space.
Then the proof of Lemma \ref{lem:policyevaluation_rl} can be applied such that 
\begin{equation}
    \label{eq:contraction}
    \|\T_h^{\pi}Q_1 - \T_h^{\pi}Q_2\|_{\infty} \leq \gamma \|Q_1 - Q_2\|_{\infty}.
\end{equation}
Thus if  $Q_{k+1} :=\T_h^{\pi}Q_k$, the series $\{Q_k\}$ converges to $Q_h^{\pi}$, i.e. $\|Q_k - Q_h^{\pi}\|_{\infty} \rightarrow 0$ as $k \rightarrow \infty$.
\end{proof}
%---------------------------
%-------------------------------
\subsection{Soft Policy Improvement}
It must be shown that \eqref{eq:policy_improvement_sac} improves the policy.

\begin{lemma}
\label{lem:policyimprovement_hrl_i} (One-step soft improvement). Let $s_t$ be fixed and $$Q_h^{\pi}(s_t,a_t) := r(s_t,a_t) + \gamma \E_{\substack{s_{t+1} \sim P(\cdot \mid s_t, a_t) \\a_{t+1} \sim \pi(\cdot \mid s_{t+1})}}[Q_h^{\pi}(s_{t+1}, a_{t+1}) - \alpha \log{\pi(a_{t+1} \mid s_{t+1})}].$$ If the policy is updated according to \eqref{eq:policy_improvement_sac}, $Q_h$ is bounded, $\alpha > 0$, $0 \leq \gamma < 1$ and $|\A| < \infty$, then
$$\E_{a_t \sim \pi_{k+1}(\cdot \mid s_t)}[Q_h^{\pi_k}(s_t,a_t) - \alpha \log{\pi_{k+1}(a_t\given s_t)}] \geq  \E_{a_t \sim \pi_{k}(\cdot \mid s_t)}[Q_h^{\pi_k}(s_t,a_t) - \alpha \log{\pi_{k}(a_t\given s_t)}].$$
\end{lemma}
\begin{proof}
\begin{eqnarray}
%\hspace{4pt}
\pi_{k+1} (\cdot \mid s_t) &=& \argmin_{\pi \in \Pi}{D_{KL} \left( \pi(\cdot\given s_t)\
\mathrel{\Big|\Big|} \ \frac{\exp{(\frac{1}{\alpha}Q_h^{\pi_k}(s_t, \cdot))}}{W^{\pi_k}(s_t)} \right)} \nonumber\\
&=& \argmin_{\pi \in \Pi} \E_{a_t \sim \pi(\cdot \mid s_t)} \left[\log\left(\frac{\pi(a_t\given s_t) W^{\pi_k}(s_t)}{\exp{\frac{1}{\alpha}Q_h^{\pi_k}(s_t, a_t)}}\right) \right]  \nonumber\\
&=& \argmin_{\pi \in \Pi} \E_{a_t \sim \pi(\cdot \mid s_t)} \left[ \log{\pi(a_t\given s_t)} - \frac{1}{\alpha} Q_h^{\pi_k}(s_t, a_t) + \log{W^{\pi_k}(s_t)} \right]. \nonumber
\end{eqnarray}
Due to minimization, it holds
$\E_{a_t \sim \pi_{k+1}(\cdot \mid s_t)} \left[ \log{\pi_{k+1}(a_t\given s_t)} - \frac{1}{\alpha} Q_h^{\pi_k}(s_t, a_t) + \log{W^{\pi_k}(s_t)} \right] \leq \E_{a_t \sim \pi_k(\cdot\mid s_t)} \left[ \log{\pi_{k}(a_t\given s_t)} - \frac{1}{\alpha} Q_h^{\pi_k}(s_t, a_t) + \log{W^{\pi_k}(s_t)} \right]$
and since the partition function does not depend on the action, it follows
$$\E_{a_t \sim \pi_{k+1}(\cdot \mid s_t)}[Q_h^{\pi_k}(s_t,a_t) - \alpha \log{\pi_{k+1}(a_t\given s_t)}] \geq  \E_{a_t \sim \pi_{k}(\cdot \mid s_t)}[Q_h^{\pi_k}(s_t,a_t) - \alpha \log{\pi_{k}(a_t\given s_t)}].$$
\end{proof}

The soft Q-value is expanded to show that $\pi_{k+1}$ is indeed an improvement on $\pi_k$ in the following Lemma. 
%------------------------
%------------------------
\begin{lemma}
\label{lem:policyimprovement_hrl_ii} (Global soft policy improvement)
Assume $|\A|<\infty$, $\alpha>0$, $0 \leq \gamma \leq 1$ and bounded $Q_h^{\pi}$.  
If  $\pi_{k+1}$ is obtained according to the soft policy improvement in \eqref{eq:policy_improvement_sac}, then 
$$Q_h^{\pi_k}(s_t, a_t) \leq Q_h^{\pi_{k+1}}(s_t, a_t) \quad\forall (s_t, a_t) \in \s \times \A, $$
when applying Lemma \ref{lem:policyimprovement_hrl_i} pointwise at each successor state.
\end{lemma}
\begin{proof}
Analogous to  Lemma \ref{lem:policyimprovement_rl}, the proof relies on expanding the soft Q-function in the Bellman equation repeatedly with actions of the improved policy.
For convenience, define $\pi' := \pi_{k+1}$, $\pi := \pi_k$ and
$\E_{\pi', P, k}[\cdot \mid s_t] :=
\E_{a_{t:k-1}\sim \pi'(\cdot\mid s_{t:k-1}), s_{t+1:k} \sim P(\cdot \mid s_{t:k-1}, a_{t:k-1})}[\cdot \mid s_t]$. 
Applying the one-step improvement from Lemma \ref{lem:policyimprovement_hrl_i}) at each successor state yields
\begin{align*}
V_h^{\pi}(s_t) 
&\leq \E_{\pi', P, t+1}[Q_h^{\pi}(s_t,a_t) -  \alpha \log{\pi'(a_t \mid s_t)} \mid s_t] \\
&= \E_{\pi', P, t+1}[r(s_t,a_t) - \alpha \log \pi'(a_t \mid s_t) + \gamma V_h^{\pi}(s_{t+1}) \mid s_t] \\
&\leq \E_{\pi', P, t+2} [ r(s_t,a_t) - \alpha \log \pi'(a_t \mid s_t) + \gamma (Q_h^{\pi}(s_{t+1}, a_{t+1}) - \alpha \log{\pi'(a_{t+1} \mid s_{t+1}))} \mid s_t] \\
&= \E_{\pi', P, t+2}[r(s_t,a_t) - \alpha \log \pi'(a_t \mid s_t) + \gamma (r(s_{t+1}, a_{t+1}) - 
\alpha \log{\pi'(a_{t+1} \mid s_{t+1})}) + \\
& \quad\quad\quad\quad \gamma^2 V_h^{\pi}(s_{t+2}) \mid s_t] \\
& \vdots \\
& \leq \E_{\pi', P, T-1}[r(s_t, a_t) - \alpha \log{\pi'(a_t \mid s_t)}+ \gamma (r(s_{t+1}, a_{t+1})- \alpha \log{\pi'(a_{t+1}\given s_{t+1})}) +  \hdots \\
& \quad\quad\quad\quad\quad\quad
\hdots + \gamma^{T-1} (r(s_{T-1}, a_{T-1}) - \alpha \log{\pi'(a_{T-1}\given s_{T-1})}) \given  s_t=s] \\
&= V_h^{\pi'}(s_t).
\end{align*}

Thus for any $s_t$
\begin{align*}
Q_h^{\pi_k}(s_t,a_t)&=\E_{  s_{t+1} \sim P(\cdot \mid s_t, a_t)}[r(s_t,a_t)+\gamma V_h^{\pi_k}(s_{t+1})] \\
&\leq \E_{s_{t+1} \sim P(\cdot \mid s_t, a_t)}[r(s_t,a_t)+\gamma V_h^{\pi_{k+1}} (s_{t+1})]=Q_h^{\pi_{k+1}}(s_t,a_t).  
\end{align*}
\end{proof}
%--------------------------
%--------------------------
\subsection{Iteration}
\label{sec:policyiteration_hrl}
\begin{lemma}
\label{lem:policyiteration_sac}
Let the reward be bounded, $|\A| < \infty$, $|\s|< \infty$, $0 \leq \gamma < 1$ and $\alpha >0$. Let $\Pi$ be a set of all stationary, stochastic policies. Alternating between exact soft policy evaluation and global soft policy improvement from some initial policy $\pi_0 \in \Pi$ and soft state-action function $Q_h^{(0)}$, the process converges (in the limit) to the optimal policy $\pi^*$, satisfying $Q_h^{\pi^*} = Q_h^*$. 
\end{lemma}
\begin{proof}
At each iteration $k$, the policy $\pi_k$ is fixed during evaluation. By Lemma \ref{lem:policyevaluation_hr}, $\T_h^{\pi_k}$ is a $\gamma$-contraction and the evaluation sequence converges to the fixed point $Q_h^{\pi_k}$. Therefore, $Q_h^{\pi_k}$ is well defined at each iteration.
By Lemma \ref{lem:policyimprovement_hrl_ii}, the sequence $\{Q_h^{\pi_k}(s_t,a_t)\}$ is non-decreasing for each pair $(s_t,a_t)$. Because the reward is bounded and $|\s| < \infty$, $\A < \infty$, the sequence converges to some point-wise limit $\bar{Q}$. Policies are updated by the softmax of the current $Q_h$ and softmax is continuous, therefore the limit point $\pi^*$ exists and is the softmax of $\bar{Q}$, i.e. $\pi^*= \text{softmax}(\bar{Q})$. 
Since $\pi^* \in \Pi$, Lemma \ref{lem:policyevaluation_hr} equally guarantees well-definedness in the limit point $Q_h^{\pi^*}$.
The map $\pi \mapsto Q_h^{\pi}$ is continuous, hence $\pi_k \rightarrow \pi^*$ implies $Q_h^{\pi_k} \rightarrow Q_h^{\pi^*}$ and therefore $\bar{Q} = Q_h^{\pi^*}$. 
It follows $\pi^* = \text{softmax}(Q_h^{\pi^*})$ (the policy is soft-greedy w.r.t. $Q_h^{\pi^*}$) and therefore $Q_h^{\pi^*}$ satisfies the Bellman optimality equation, thus $Q_h^{\pi^*} = Q_h^*$.
\end{proof}

\section{Hyperparameters}
\label{sec:hyperparameters}
Hyperparameters for the C-DSAC experiments were systematically tuned to maintain parity with the SAC baselines, ensuring an equitable comparison. 
\begin{table}[H]
\centering
\caption{C-DSAC Hyperparameters}
\begin{tabular}{ll}
\hline
\textbf{Parameter} & \textbf{Value} \\ \hline
\textit{\textbf{Shared}} &  \\
Optimizer & Adam \citep{Kingma2014AdamAM} \\
Learning rate & $3 \cdot 10^{-4}$ \\
Discount factor ($\gamma$) & 0.99 \\
Replay buffer size & $10^6$ \\
Actor architecture & (256, 256) \\
Minibatch size & 256 \\
Layer activation & ReLU \\
Target smoothing coefficient ($\tau$)  & 0.005 \\
Target update interval & 1 \\
Gradient steps & 1 \\ \hline
\textit{\textbf{SAC}} & \\
Critic architecture & (256, 256) \\ \hline
\textit{\textbf{C-DSAC}} & \\
Critic architecture & (256, 255) \\
Min., Max. Std. & 0.01, 1000 \\
Gauss grid points & 31 \\
Gauss bounds & $15\sigma$
\label{tab:cdsac_hyper}
\end{tabular}
\end{table}

\begin{table}[H]
\centering
\caption{Environment-specific Entropy}
\begin{tabular}{cc}
\hline
\textbf{Env. Name} & \textbf{$\alpha$-value} \\ \hline
Hopper-v4 & $0.2$ \\
Ant-v4 & $0.2$ \\
HalfCheetah-v4 & $0.2$ \\
Walker2d-v4 & $0.2$ \\
Humanoid-v4 & $0.05$
\label{tab:env_entropy}
\end{tabular}
\end{table}

\bibliographystyle{plain}
\bibliography{RL.bib}

\end{document}